\RequirePackage{amsmath}
\documentclass[runningheads]{llncs}
\usepackage[T1]{fontenc}
\usepackage{graphicx}

\usepackage[hidelinks]{hyperref}
\usepackage{color}

\usepackage{amssymb}
\usepackage{amsfonts}
\usepackage{xspace}
\usepackage{enumitem}
\usepackage[normalem]{ulem}
\usepackage[noend,ruled,linesnumbered]{algorithm2e}
\usepackage{placeins}
\SetKwInOut{Input}{input}
\SetKwInOut{Output}{output}
\let\oldnl\nl
\newcommand{\nlnonumber}{\renewcommand{\nl}{\let\nl\oldnl}}

\usepackage{array}
\usepackage{booktabs}
\usepackage{makecell}
\usepackage{multirow}
\usepackage{tikz}
\usetikzlibrary{arrows.meta, shapes, positioning}

\usepackage{float}

\usepackage{pgfplots}
\pgfplotsset{compat=1.18}

\newcommand{\A}{\mathcal{A}}
\newcommand{\C}{\mathcal{C}}
\newcommand{\E}{\mathcal{E}}

\newcommand{\I}{\mathcal{I}}
\renewcommand{\L}{\mathcal{L}}
\renewcommand{\O}{\mathcal{O}}
\renewcommand{\P}{\mathcal{P}}
\newcommand{\R}{\mathcal{R}}
\newcommand{\T}{\mathcal{T}}

\newcommand{\alphabet}{\mathrm{\Sigma}}
\newcommand{\predSet}{\alphabet_P}
\newcommand{\conceptSet}{\alphabet_C}
\newcommand{\roleSet}{\alphabet_R}
\newcommand{\constSet}{\alphabet_I}

\newcommand{\varSet}{\alphabet_V}

\newcommand{\Const}{\textit{Const}}

\newcommand{\ISA}{\sqsubseteq}
\newcommand{\theo}{\mathrm{\Phi}}

\newcommand{\aczero}{\mathrm{AC}^0}

\newcommand{\NP}{\textrm{NP}}

\newcommand{\PTIME}{\textrm{PTIME}}

\newcommand{\pidue}{\mathrm{\Pi}^p_2}

\newcommand{\policyclass}{C_p}
\newcommand{\policyarb}{\mathsf{P_A}}
\newcommand{\policylin}{\mathsf{P_L}}

\newcommand{\tup}[1]{\langle #1 \rangle}
\newcommand{\vseq}[1]{\vec{#1}}
\newcommand{\ra}{\rightarrow}

\newcommand{\atomrewr}{\mathsf{AtomRewr}}

\newcommand{\K}{\mathrm{K}}

\newcommand{\decprob}[1]{\textsc{#1}\xspace}
\newcommand{\GAEnt}{\decprob{GA-Ent}}
\newcommand{\IGAEnt}{\decprob{IGA-Ent}}
\newcommand{\CoreGAEnt}{\decprob{MPV-Ent}}

\newcommand{\modelsga}{\models_\mathsf{GA}}
\newcommand{\modelsiga}{\models_\mathsf{IGA}}
\newcommand{\modelsmpv}{\models_\mathsf{MPV}}
\newcommand{\modelscore}{\modelsmpv}
\newcommand{\optcens}{\mathsf{optCens}}

\newcommand{\dlliter}{\text{DL-Lite}_\R}
\newcommand{\EL}{\mathcal{EL}}
\newcommand{\RL}{\mathcal{RL}}

\renewcommand{\Im}{\textit{Im}}
\newcommand{\im}{\Im}
\newcommand{\gs}{\textit{gs}}

\newcommand{\body}{\mathsf{body}}
\newcommand{\head}{\mathsf{head}}

\newcommand{\eval}{\textit{eval}}

\newcommand{\true}{\texttt{true}}
\newcommand{\false}{\texttt{false}}

\newcommand{\censiga}{\C_{\mathsf{IGA}}^{\E}}
\newcommand{\censmpv}{\C_{\mathsf{MPV}}^{\E}}
\newcommand{\mpvcens}{\censmpv}
\newcommand{\corecens}{\mpvcens}

\newcommand{\clt}[1]{\mathsf{cl}_{\T}(#1)}
\newcommand{\modelseql}{\models_{\mathsf{EQL}}}

\newcommand{\qedex}{\hfill\ensuremath{\vrule height 4pt width 4pt depth 0pt}}

\newcommand{\disclosablealg}{\mathsf{Disclosable}}

\newcommand{\igaentalg}{\mathsf{IGA\text{-}Entails}}

\newcommand{\mpventdlliteralg}{\mathsf{MPV\text{-}Entails\text{-}DL\text{-}Lite_\R}}
\newcommand{\ampventdlliteralg}{\mathsf{AMPV\text{-}Entails\text{-}DL\text{-}Lite_\R}}

\newcommand{\constSymbol}[1]{\mathsf{#1}}
\newcommand{\predSymbol}[1]{\mathsf{#1}}
\newcommand{\minor}{\predSymbol{M}}

\newcommand{\disease}{\predSymbol{Dis}}
\newcommand{\pediatrician}{\predSymbol{P}}
\newcommand{\affectedBy}{\predSymbol{affBy}}
\newcommand{\treatedBy}{\predSymbol{trBy}}
\newcommand{\parentalCons}{\predSymbol{pConsFor}}
\newcommand{\parentOf}{\predSymbol{parOf}}
\newcommand{\bioParOf}{\predSymbol{bioParOf}}
\newcommand{\ann}{\constSymbol{ann}}
\newcommand{\bob}{\constSymbol{bob}}
\newcommand{\cloe}{\constSymbol{cloe}}
\newcommand{\disConst}{\constSymbol{d}}

\newenvironment{proofsk}{\noindent \textsl{Proof (sketch).\ }}{
}

\newif\ifdraft

\newif\ifshowproofs
\showproofstrue 

\newif\ifcameraready
\camerareadytrue

\newif\iflong
\longtrue 

\iflong\showproofstrue\fi

\usepackage{environ}
\NewEnviron{hiddenproof}{}
\ifshowproofs\else

\fi

\ifdraft
    \usepackage[draft,colorinlistoftodos]{todonotes}
    \newcommand{\nb}[1]{%
        {
        \todo[tickmarkheight=0.5em,size=\scriptsize,color=red!30,linecolor=red!70]{#1}}}
    \newcommand{\alert}[1]{{\color{red} #1}}
\else
    \newcommand{\nb}   [1]{}
    \newcommand{\todo} [1]{}
    \newcommand{\alert}[1]{}
\fi

\begin{document}
    \title{%
    \iflong
        Tractable Query Answering under Epistemic Confidentiality Policies in DL Ontologies (extended version)
    \else
        Tractable Query Answering under Epistemic Confidentiality Policies in Description Logic Ontologies
    \fi}
    %
    \titlerunning{Tractable Query Answering under Epistemic Confidentiality Policies in DL Ontologies}
    \author{\ifcameraready%
        Lorenzo Marconi
        \orcidID{0000-0001-9633-8476} 
        \and
        Daniela Rieti
        \and
        \newline
        Riccardo Rosati
        \orcidID{0000-0002-7697-4958}
    \else
        Anonymous authors
    \fi}
    \authorrunning{\ifcameraready%
        L.\ Marconi et al.
    \fi}
    %
    \institute{\ifcameraready%
        Sapienza University of Rome, Italy\\
        \email{\{marconi,rosati\}@diag.uniroma1.it}\\
        \email{rieti.1762973@studenti.uniroma1.it}
    \fi}
    \maketitle              
    \begin{abstract}
        We study Controlled Query Evaluation (CQE), a declarative approach to confidentiality-preserving data access, in the context of Description Logic (DL) 
        ontologies, and for confidentiality policies expressed through Epistemic Dependencies (EDs).
        We first address the problem of answering queries (specifically, Boolean unions of conjunctive queries) under known semantics for CQE (GA- and IGA-entailment). Our results show that if the TBox is expressed in $\dlliter$, CQE is computationally intractable in general.
        Moreover, in the presence of EDs, the IGA semantics has recently been proven not to satisfy an important confidentiality preservation property known as \emph{indistinguishability}.
        With the goal of defining computationally easier and confidentiality-preserving forms of CQE, we introduce a new semantics for CQE, based on the notion of \emph{minimal policy violation} (MPV). We show that the new semantics provides a sound approximation of the previous ones, while satisfying the indistinguishability property.
        We also prove that, in the case of $\dlliter$ ontologies, query entailment under the MPV semantics 
        can be decided in polynomial time in data complexity. 
        Finally, we present a software implementation of our framework that we used to evaluate the feasibility of this new approach using an existing benchmark for OWL~2~QL.
        
    \end{abstract}
    
    \section{Introduction}
\label{sec:introduction}

We study \emph{Controlled Query Evaluation} (CQE)~\cite{SiJR83,BiBo04}, a logical framework designed to allow answering queries over a database or knowledge base while keeping undisclosed pieces of knowledge that are considered sensitive. 
In recent years, several works on CQE focused on Description Logic ontologies~\cite{BoSa13,CKKZ13,CLRS24}, which is the setting we consider here.
In particular, we focus on ontologies whose intentional part is expressed in $\dlliter$~\cite{CDLLR07}, the logic underpinning the OWL~2~QL profile~\cite{W3Crec-OWL-Profiles}.

In CQE, the information to be protected is specified using logical formulas\iflong, which are \else\ \fi collected in the so-called \emph{policy}.
In the literature, policies are often expressed by means of \emph{denial assertions}, i.e.\ logical sentences of the form $q\ra\bot$, where $q$ is a Boolean conjunctive query (BCQ). Intuitively, such assertions are intended to prevent the disclosure of $q$ to the end user, even indirectly through an unbounded sequence of queries. 
The recent work~\cite{CLMRS24} proposed the usage of \emph{epistemic dependencies} (EDs) for enhancing the expressivity of the data protection policy. An ED is a logical formula of the form $\forall \vseq{x}\,(\K q_b\ra\K q_h)$, where both $q_b$ and $q_h$ are (possibly open) CQs whose free variables occur in $\vseq{x}$, and $\K$ denotes an epistemic operator. Intuitively, for every instantiation of the variables $\vseq{x}$, such a formula prevents the user from inferring $q_b$ unless $q_h$ can also be inferred.

In CQE, the key notion used to represent disclosable information is called \emph{censor}. A censor is a set of logical sentences that, even when combined with the intensional knowledge provided by the TBox, cannot be used to infer sensitive information. Such sentences may be expressed in different formalisms, like the language of BCQs (as in the aforementioned work~\cite{CLMRS24}) or ground atoms. Specifically, censors consisting of ground atoms (also referred to as \emph{GA censors}) provide a natural representation of disclosable information, since they can be viewed as an ABox.
GA censors are called \emph{optimal} when maximal w.r.t.\ set inclusion.

\begin{example}
    \label{ex:intro}
    A hospital does not want to disclose the fact that a minor ($\minor$) is affected by ($\affectedBy$) some disease.
    Furthermore, the fact that someone is a minor can be disclosed only with parental consent ($\parentalCons$).
    Finally, parental relationships ($\parentOf$) may be revealed only if they are biological ($\bioParOf$).
    Consider a policy $\P$ made of three EDs, respectively encoding the above protection rules:
    \iflong
    \[
    \begin{array}{r@{\,}l}
        \forall x,y & (\K(\minor(x)\land\affectedBy(x,y)) \ra \bot) \\
        \forall x & (\K\minor(x) \ra \K\exists y\,\parentalCons(y,x)) \\
        \forall x,y & (\K\parentOf(y,x) \ra \K\bioParOf(y,x))
    \end{array}
    \]
    \else
    \[
    \begin{array}{r@{\,}l@{\qquad}r@{\,}l}
        \forall x,y & (\K(\minor(x)\land\affectedBy(x,y)) \ra \bot) &
        \forall x & (\K\minor(x) \ra \K\exists y\,\parentalCons(y,x)) \\
        \multicolumn{4}{c}{
            \forall x,y \, (\K\parentOf(y,x) \ra \K\bioParOf(y,x))
        }
    \end{array}
    \]
    \fi
    The intensional knowledge is provided by means of a DL TBox $\T$ allowing to infer that if an individual $x$ gave a parental consent for $y$, then $x$ is a parent of $y$ ($\parentalCons\ISA\parentOf$), and that if $x$ is affected by $y$, then $y$ is a disease ($\exists\affectedBy^-\ISA\disease$). Finally, the ground data is contained in the DL ABox $\A$, which contains the fact that $\ann$ is a minor, she is affected by the disease $\disConst$, and $\bob$ gave his consent to reveal that $\ann$ is a minor: $\A=\{\minor(\ann),\allowbreak\affectedBy(\ann,\disConst),\allowbreak\parentalCons(\bob,\ann)\}$.
    Note that the facts $\parentOf(\bob,\ann)$ and $\disease(\disConst)$ are 
    also logical consequences of $\T\cup\A$.
    Since the fact that $\bob$ is the biological parent of $\ann$ cannot be revealed (because it is not entailed by the ontology), then also $\parentalCons(\bob,\ann)$ (which implies $\parentOf(\bob,\ann)$ via $\T$) must be kept undisclosed; in turn, the same applies to $\minor(\ann)$. In this case, we have only one optimal GA censor, i.e.\ $\{\affectedBy(\ann,\disConst),\disease(\disConst)\}$, containing the only two facts that can be revealed.
    
    Differently, considering a policy $\P'$ containing instead only the first ED, two optimal GA censors would exist, consisting of all the facts that are logical consequences of the ontology except, respectively, $\minor(\ann)$ and $\affectedBy(\ann,\disConst)$.
    \qedex
\end{example}

The paper~\cite{MRR25} studied the properties of GA censors in the presence of EDs and the complexity of entailment of Boolean unions of conjunctive queries (BUCQs) under the so-called \emph{GA-} and \emph{IGA-entailment} semantics, which we also use as a baseline for our analysis. The former consists of checking whether every optimal GA censor, together with the given TBox, logically entails the input query; the latter, on the other hand, checks whether the query is entailed by the TBox and the intersection of all the optimal GA censors.
That paper focused on $\dlliter$ ontologies and (combinations of) different subclasses of EDs, namely linear, full, and acyclic dependencies, showing nice computational properties for IGA-entailment (in the case where EDs are either linear-full or acyclic-full), but lacking the whole expressivity offered by EDs in their general form. 

In this paper, for $\dlliter$ ontologies and arbitrary EDs, we show that both GA- and IGA-entailment of BUCQs are intractable: specifically, they are $\pidue$-complete in data complexity, with hardness already holding even for an empty TBox.
Arbitrary EDs have a second important impact on the above entailment semantics: they do not enjoy a crucial property related to confidentiality preservation, known as \emph{indistinguishability}~\cite{BiBo04}.
This property, already studied for EDs in~\cite{CLMRS24}, ensures that it is always possible to 
provide an ABox that contains no sensitive information, while guaranteeing that the CQE system with this new ABox as input behaves identically to the original one for every possible query.

To overcome these limitations, we first continue investigating the linear case (i.e.\ the case where only one atom occurs in $q_b$), without \iflong imposing \fi the additional restrictions of~\cite{MRR25}.
In this scenario, we observe that both GA- and IGA-entailment satisfy the indistinguishability property, and we show that both the BUCQ entailment problems become tractable in data complexity for $\dlliter$ ontologies.
Linear EDs alone are, however, rather limited for expressing censoring rules: they cannot even capture denial assertions, \iflong which serve as \fi the \textit{de facto} baseline in most CQE studies\iflong, since they represent \else\ and \fi one of the most natural ways to specify protection rules.

Thus, we propose a new approach, based on the notion of \emph{minimal policy violation} (MPV). 
Each MPV is a minimal set of facts that are logical consequences of $\T\cup\A$ and that cause the violation of at least one ED in a given set of facts $\A'$. Starting from a set $\A'$ containing all the logical consequences of $\T\cup\A$, 
we iteratively remove all occurring MPVs. The fixpoint of such an iterative approach results in a GA censor, which we call an \emph{MPV censor}.
We show that the resulting notion of MPV-entailment: $(i)$ soundly approximates GA- and IGA-entailment and notably, coincides with IGA-entailment when the dependencies are linear and the TBox is in $\dlliter$, or when the policy consists of denials; $(ii)$ satisfies the indistinguishability property; $(iii)$ is tractable (\PTIME-complete) in data complexity for $\dlliter$. 

Finally, we present an experimental evaluation of MPV-entailment, focusing on the tractable setting identified in this paper, namely $\dlliter$. Unlike existing implementations~\cite{BC25,MRR25}, which mainly rely on query rewriting techniques, \iflong our approach involves the construction of \else we construct \fi the MPV censor through suitable SQL-based data manipulation queries. Once the MPV censor is computed, query entailment is verified using standard techniques.
The experiments are based on OWL2Bench~\cite{SBM20}, a benchmark that supports the generation of custom-sized ABoxes, which allowed us to test the scalability of our approach w.r.t.\ the amount of ground data.

The paper is organized as follows. 
After recalling preliminary notions, in Section~\ref{sec:framework} we formally present our CQE framework. 
Then, in Section~\ref{sec:theoretical-results} we analyze the computational properties of GA- and IGA-entailment in our framework. 
In Section~\ref{sec:mpv} we introduce the new MPV semantics for CQE and study the computational properties of 
MPV-entailment for lightweight DLs.
In Section~\ref{sec:experiments} we present our experimental evaluation of MPV-entailment in the context of $\dlliter$ ontologies. 
We conclude in Section~\ref{sec:conclusions}.

\ifcameraready
\else
\fi
    \section{Preliminaries}
\label{sec:preliminaries}
We refer to standard notions of First-Order (FO) Logic and Description Logic (DL).
We consider an alphabet $\alphabet$ of symbols partitioned into three countably infinite subsets $\predSet$, $\varSet$, and $\constSet$, used respectively for denoting predicates, variables, and constant symbols (or individuals). In turn, $\predSet$ is partitioned into two mutually disjoint sets $\conceptSet$ and $\roleSet$, containing \emph{concept} and \emph{role} names, i.e., unary and binary predicates. 
An \emph{atom} $\alpha$ is a formula of the form $P(\vseq{t})$ where, $P\in\predSet$ and $\vseq{t}$ is a sequence of \emph{terms} (i.e.\ symbols from $\constSet\cup\varSet$). 
$\alpha$ is called \emph{ground} (or a \emph{fact}) if $t_i\in\constSet$ for every $t_i\in\vseq{t}$.
An FO \emph{sentence} is a formula without free variables. To make explicit the free variables $\vseq{x}$ of a formula $\phi$, we write $\phi(\vseq{x})$.
Given an FO theory (set of sentences) $\theo$, we denote by $\Const(\theo)$ the set of constants occurring in the formulas of $\theo$.
We write $\eval(\phi,\I)$ to indicate the evaluation of an FO sentence $\phi$ over an FO interpretation $\I$.
A \emph{model} of an FO theory $\Phi$ is an FO interpretation satisfying all sentences in $\Phi$. We say that $\Phi$ \emph{entails} an FO sentence $\phi$, denoted by $\Phi \models \phi$, if $\eval(\phi,\I)$ is true in every model $\I$ of $\Phi$.

We use the term \emph{query} as a synonym of FO formula. We consider \emph{conjunctive queries} (CQs), i.e., queries of the form $q=\exists \vseq{x}\,\phi(\vseq{y})$, where $\phi$ is a conjunction of atoms and $\vseq{x}\subseteq\vseq{y}$. When $\vseq{x}=\vseq{y}$, we call $q$ a \emph{Boolean conjunctive query} (BCQ).
\emph{(Boolean) unions of conjunctive queries}, or (B)UCQs, are disjunctions of (Boolean) conjunctive queries $q_1(\vseq{y})\lor\ldots\lor q_k(\vseq{y})$. 
We also consider the special ground CQ $\bot$ assuming that $\eval(\bot,\I)$ is false for every FO interpretation $\I$.

A \emph{DL ontology} is an FO theory $\O=\T\cup\A$, where $\T$ is called \emph{TBox} and $\A$ is called \emph{ABox}. Every ABox is a set of facts, 
whereas the shape of a TBox depends on the specific DL considered. The main complexity results provided in this work focus on $\dlliter$~\cite{CDLLR07}. 
In such a DL, a \emph{basic role} $R$ is either a symbol $P\in\roleSet$, or its \emph{inverse} $P^-$.
A \emph{basic concept} $C$ is either a symbol of $\conceptSet$ or a so-called \emph{unqualified existential restriction} $\exists R$, for some basic role $R$.
Then, a $\dlliter$ TBox $\T$ is a finite set of inclusion and disjointness axioms between basic concepts or basic roles. Specifically, these assertions take the forms $C_1 \ISA C_2$, $C_1 \ISA \neg C_2$, $R_1 \ISA R_2$, and $R_1 \ISA \neg R_2$, where $C_1, C_2$ are basic concepts and $R_1, R_2$ are basic roles.


Given an ontology $\O = \T\cup\A$ we write $\clt{\A}$ to denote the set of all ground atoms $\alpha$ such that $\T\cup\A\models\alpha$.
Moreover, given a BUCQ $q$ and an ontology $\T\cup\A$, we call \emph{image of $q$ in $\A$ w.r.t.\ $\T$} any $\subseteq$-minimal subset $I\subseteq\A$ such that $\T\cup I\models q$. We denote by $\im(q,\A,\T)$ the set of all such images.

All our complexity results refer to data complexity~\cite{V82}, i.e.\ 
the complexity with respect to the size of the ABox.
In the case of $\dlliter$, standard entailment of a BUCQ is in $\PTIME$ (actually in $\aczero$~\cite{CDLLR07}) in data complexity.

    \section{CQE Framework}
\label{sec:framework}

In CQE, data protection rules are expressed in terms of logical formulas. In particular, we adopt epistemic dependencies~\cite{CLMRS24}, which are EQL-Lite(CQ)~\cite{CDLLR07b} sentences defined as follows.

\begin{definition}[Epistemic dependency]
\label{def:epistemic-dependency}
    An \emph{epistemic dependency (ED)} is a sentence $\tau$ of the \iflong following \fi form
    \begin{center}
        $\forall \vseq{x}_1,\vseq{x}_2\,(\K q_b(\vseq{x}_1,\vseq{x}_2) \rightarrow \K q_h(\vseq{x}_2))$
    \end{center}
    where $q_b(\vseq{x}_1,\vseq{x}_2)$ 
    and $q_h(\vseq{x}_2)$ 
    are CQs, and
    $\K$ denotes an epistemic operator.\iflong\footnote{Observe that all the free variables of $q_b$ and $q_h$ are universally quantified outside the scope of the implication.}\fi\ 
    Given an ED $\tau$ 
    of the above form, $\body(\tau)$ denotes $q_b$ and $\head(\tau)$ denotes $q_h$.
\end{definition}
Then, we call a finite set of EDs a \emph{protection policy} (or simply \emph{policy}).
An ED $\tau$ is called \emph{linear} if $\body(\tau)$ contains only one atom. We indicate with $\policyarb$ and $\policylin$ the classes of policies consisting, respectively, of arbitrary and linear EDs. 
Moreover, given an ED $\tau$, we call \emph{ground substitution} for $\tau$ any mapping of its universally quantified variables to constants.
Then, $\gs(\tau)$ denotes the set of all the ground substitutions for $\tau$ and, given a set $\Gamma$ of constants, $\gs(\tau,\Gamma)$ denotes the finite subset of $\gs(\tau)$ of substitutions over the constants of $\Gamma$.

An FO theory $\theo$ is said to \emph{satisfy} an ED $\tau$ (written $\theo \modelseql \tau$) if, for every $\sigma\in\gs(\tau)$, whenever $\theo \models \body(\sigma(\tau))$ holds, then also $\theo \models \head(\sigma(\tau))$ holds.
If $\theo$ satisfies all EDs in a policy $\P$, we say that $\theo$ \emph{satisfies} $\P$, and write $\theo \modelseql \P$.

We remark that the above definition of satisfaction exploits the following property of the epistemic operator $\K$ in EQL: for every FO theory $\theo$ and for every formula of the form $\K\phi$ such that $\phi$ is an FO sentence, $\theo\modelseql \K\phi$ iff $\theo\models\phi$. Hence, the epistemic operator allows for expressing entailment with respect to an FO theory.
The next definitions exploit such a property so that, through the use of the epistemic operator $\K$ in EDs, the disclosability of a formula will be strictly related to the notion of FO entailment of the formula in the ontology.


The triple $\E=\tup{\T,\P,\A}$, called a \emph{CQE instance}, contains the three main components of our framework. Here, $\T$ is a DL TBox, $\P$ is a policy, and $\A$ is an ABox such that $\T \cup \A$ is consistent. To emphasize that $\T$ is a TBox for a certain DL $\L$, we also call $\E$ an \emph{$\L$ CQE instance}.

In general, it may happen that $\T \cup \A \not\modelseql \P$, meaning that the ontology does not comply with the policy. To prevent the disclosure of protected information, we adopt the notion of GA censor~\cite{CLRS24}, which models a piece of ground knowledge that can be disclosed without violating the specified protection policy.

\begin{definition}[GA censor]
\label{def:ga-censor}
    Given a CQE instance $\E = \tup{\T,\P,\A}$, a \emph{ground atom (GA) censor} of $\E$ is any subset $\C \subseteq \clt{\A}$ such that $\T \cup \C \modelseql \P$.
    Furthermore, we call $\C$ \emph{optimal} if no set $\C' \supset \C$ is a GA censor of $\E$.
    We denote by $\optcens(\E)$ the set of all the optimal GA censors of $\E$.
\end{definition}

Since a given CQE instance may, in general, admit multiple optimal GA censors, a policy-protected query answering framework must be equipped with a formal semantics specifying how 
censors are to be used in determining query entailment.
For GA censors, the most studied \iflong query answering \fi semantics are the following ones.

\begin{definition}[GA- and IGA-entailment]
    For a CQE instance $\E=\tup{\T,\P,\A}$ and a BUCQ $q$, we say that $q$ is
    \begin{itemize}
        \item \emph{GA-entailed} by $\E$ (in symbols, $\E\modelsga q$) if $\T\cup\C\models q$ for each $\C\in\optcens(\E)$;
        \item \emph{IGA-entailed} by $\E$ (in symbols, $\E\modelsiga q$) if $\T\cup\censiga\models q$, where $\censiga$ denotes the set $\bigcap_{\C\in\optcens(\E)}\C$.
    \end{itemize}
\end{definition}

We now introduce
the decision problems related to the 
semantics defined above, parameterized with respect to a DL $\L$ and a class of policies $\policyclass$.
We define $\GAEnt[\L,\policyclass]$ (resp., $\IGAEnt[\L,\policyclass]$) as the problem of verifying whether, given an $\L$ CQE instance $\E=\tup{\T,\P,\A}$ such that $\P\in\policyclass$ and a BUCQ $q$, $\E\modelsga q$ (resp., $\E\modelsiga q$).

Finally, we recall a fundamental property for CQE, namely indistinguishability~\cite{BiBo04}, \iflong which has \fi already been studied for EDs in~\cite{CLMRS24}. 
When enjoyed, this property ensures that a CQE system produces the same answers as a system whose ABox $\A'$ is free of sensitive information.

\begin{definition}[Indistinguishability]
\label{def:indistinguishability}
    \newcommand{\modelsblank}{\models_{\Diamond}}
    Let $\L$ be a DL and $\policyclass$ be a class of policies.
    A CQE entailment relation $\modelsblank$ satisfies the \emph{indistinguishability} property for $\L$ and $\policyclass$ if, for every $\L$ CQE instance $\E=\tup{\T,\P,\A}$ such that $\P\in\policyclass$, there exists a CQE instance $\E'=\tup{\T,\P,\A'}$ such that $\T\cup\A'\modelseql\P$ and, for every BUCQ $q$, $\E\modelsblank q$ iff $\E'\modelsblank q$.
\end{definition}

    \section{GA- and IGA-Entailment}
\label{sec:theoretical-results}

\newcommand{\chase}{\textit{chase}}
\newcommand{\lpv}{\textit{LPV}}
\newcommand{\lpvz}{\textit{LPV}_0}

We now study the data complexity of GA- and IGA-entailment in our CQE framework, first for arbitrary policies and then for linear policies.

\iflong
\subsection{Arbitrary Epistemic Dependencies}
\else
\subsubsection*{Arbitrary Epistemic Dependencies}
\fi
\label{sec:ga-iga-arbitrary}


We start by showing that both IGA- and GA-entailment are hard for the second level of the polynomial hierarchy, independently of the chosen DL.
This can be shown by adapting the proof of an analogous result provided in~\cite[Thm.~4.7]{CM05} in the context of database repairs.

\begin{lemma}
\label{lem:linear+denial-iga-ga-lb}
    For every DL $\L$, both $\GAEnt[\L,\policyarb]$ and $\IGAEnt[\L,\policyarb]$ are $\pidue$-hard in data complexity.
\end{lemma}
\iflong\begin{proof}
    \newcommand{\tval}{\mathsf{t}}
    \newcommand{\fval}{\mathsf{f}}
    \newcommand{\myconstant}{\mathsf{a}}
    \newcommand{\pos}{\mathit{Pos}}
    \newcommand{\pol}{\mathit{Pol}}
    \newcommand{\pconst}{\mathsf{p}}
    \newcommand{\Ix}{I_{\vseq{x}}}
    \newcommand{\Iy}{I_{\vseq{y}}}
    We prove the thesis by showing a reduction from 2-QBF, inspired by the proof of Theorem~4.7 of~\cite{CM05}. We assume that the given formula is in CNF.
    
    Let $\T=\emptyset$, and let $\P$ consist of the following EDs:
    \[ 
    \begin{array}{r@{\,}l}
        \forall v,t & (\K T(v,t) \ra \K\exists p\, \pos(v,p)) \\
        \forall v,p & (\K \pos(v,p) \ra \K\exists i\, C(p,i)) \\
        \forall p,t & (\K \pol(p,t) \ra \K\exists i\, C(p,i)) \\
        \forall v,i & (\K C(v,i) \ra \K\exists j\, S(i,j)) \\
        \forall i,j & (\K S(i,j) \ra \K\exists v\, C(v,j)) \\
        \forall i,j & (\K S(i,j) \ra \K\exists p,v,t\, (C(p,i) \land \pos(v,p) \land \\&\hspace{9.5em} \pol(p,t) \land T(v,t))) \\
        \forall v & (\K(T(v,\tval) \wedge T(v,\fval)) \ra \K\bot) .
    \end{array}
    \]
    Intuitively, the first four EDs impose what follows: every variable having a truth assignment must occur in some position; every position in which a variable $v$ occurs (or having polarity $t$) must occur in some clause; every clause must have a successor.
    The combination of the fourth and fifth EDs forces a loop of $S$-facts to collapse if one of them is missing.
    The sixth ED imposes that, if a clause $\psi_i$ has a successor $\psi_j$, then a variable $v$ must occur in $\psi_i$ such that the truth value of $v$ matches its polarity in $\psi_i$ (i.e., $\psi_i$ is satisfied). 
    Finally, the seventh ED guarantees that distinct truth values are not assigned to the same variable.
    
    Now let $\phi=\forall \vseq{x}\,\exists \vseq{y}\,(\psi_1\wedge\ldots\wedge\psi_n)$ be a 2-QBF, where every $\psi_i$ is a clause over the propositional variables 
    $\vseq{x}\cup\vseq{y}$.
    Moreover, $\A=\A_1\cup\A_2\cup\{S(1,\myconstant)\}$, where:
    \[
    \begin{array}{r@{\;}l@{}l@{}l}
        \A_1=&\{& S(0,0)\} \,\cup \\
            &\{& C(\pconst_x^0,0), \pos(\pconst_x^0,0), \pol(\pconst_x^0,\tval), \pol(\pconst_x^0,\fval), 
            T(x,\tval), T(x,\fval) \mid x\in\vseq{x} \}\\
        \A_2=&\{& S(i,(i \bmod n)+1) \mid 1\le i\le n
        \} \,\cup \\
           &\{& C(\pconst_v^i,i), \pos(v,\pconst_v^i), \pol(\pconst_v^i,\tval), T(v,\tval) \mid 
           v \textrm{ occurs positively in } \psi_i \} \,\cup \\
           &\{& C(\pconst_v^i,i), \pos(v,\pconst_v^i), \pol(\pconst_v^i,\fval), T(v,\fval) \mid 
           v \textrm{ occurs negatively in } \psi_i \} .
    \end{array}
    \]
    We prove that $\E \modelsga S(1,\myconstant)$ iff $\phi$ is valid, where $\E=\tup{\T,\P,\A}$.
    
    $(\Leftarrow)$
    First, for every interpretation $\Ix$ of $\vseq{x}$, let 
    \[
        \A(\Ix) = \A_1 \setminus \{T(x,\fval) \mid x\in\Ix\} \setminus \{T(x,\tval) \mid x\in\vseq{x}\setminus\Ix\}.
    \]
    
    Now, let us assume that $\phi$ is not valid, and let $\Ix$ be an interpretation of $\vseq{x}$ such that there exists no interpretation $\Iy$ of $\vseq{y}$ such that $\Ix\cup\Iy$ satisfies $\psi_1\wedge\ldots\wedge\psi_n$.
    We prove that $\A(\Ix)$ is an optimal GA censor of $\E$. Since $\T\cup\A(\Ix)\modelseql\P$, we only have to show that no further fact from $\clt{\A}$ (i.e.\ from $\A$) can be added to it without violating the policy $\P$.
    
    First, notice that adding any fact from $\A_1$ to $\A(\Ix)$ would violate the last dependency of $\P$.
    As for the $T$-facts contained in $\A_2$, adding one of them would either violate the last dependency as well, or cause the addition of at least one $C$-fact from $\A_2$ (because of the first dependency).
    However, if any fact from $\A_2\cup\{S(1,\myconstant)\}$ with predicate $C$, $S$, $\pol$, or $\pos$ is added to $\A(\Ix)$, then the second, third, fourth and fifth dependencies would eventually require to add, for every clause $\psi_i$, the corresponding fact $S(i,(i \bmod n)+1)$ from $\A_2$.
    Then, because of the sixth dependency, we would need to add, for every $\psi_i$, four facts of the form $C(p,i)$, $\pos(v,p)$, $\pol(p,t)$, $T(v,t)$ from $\A_2$ for at least one variable $v$ occurring in 
    $\psi_i$.
    However, due again to the last dependency of $\P$ and by construction of $\A_2$, this could be done (while preserving consistency) only if it is possible to find a truth assignment for the variables $\vseq{y}$ such that at least one variable for each clause $\psi_i$ is assigned to a value corresponding to its polarity in $\psi_i$, i.e.\ only if $\phi$ is valid, which is a contradiction.
    Consequently, $\A(\Ix)\in\optcens(\E)$, which implies that $\E\not\modelsga S(1,\myconstant)$.
    
    $(\Rightarrow)$
    Conversely, assume that $\phi$ is valid. Let $\Ix$ be any interpretation of $\vseq{x}$. Then, there exists an interpretation $\Iy$ of $\vseq{y}$ such that $\Ix\cup\Iy$ satisfies $\psi_1\wedge\ldots\wedge\psi_n$.
    Similarly as above, this is possible only if there exists a subset $\A'$ of $\A_2$ containing, for every $1\le i\le n$, at least five facts of the form $S(i,(i \bmod n)+1)$, $C(p,i)$, $\pos(v,p)$, $\pol(p,t)$, $T(v,t)$ 
    and such that, for every variable $v$, it does not contain both $T(v,\tval)$ and $T(v,\fval)$.
    
    It is immediate to verify that $\T\cup\A'\modelseql\P$.
    Now, suppose that $S(1,\myconstant)\notin\C$ for some $\C\in\optcens(\E)$. 
    Since $\C$ is optimal, this is possible only if adding $S(1,\myconstant)$ violates the fourth dependency, i.e.\ if no fact of the form $C(\_,1)$ of $\A$ occurs in $\C$.
    Then, because of the fifth dependency, the fact $S(n,1)$ (i.e., the only fact of $\A$ of the form $S(\_,1)$) does not belong to $\C$.
    Analogously, $\C$ cannot contain any fact of the form $C(\_,n)$, nor $S(n-1,n)$, and so on: considering the effect of the first three dependencies as well, this iterative process forces \emph{all} facts with predicate $\pos$, $\pol$, $C$ and $S$ occurring in $\A_2$ to be excluded from $\C$, other than all the facts of the form $T(y,\_)$ such that $y\in\vseq{y}$. Note then that $\C\subseteq\A_1$, and recall that $\T\cup\C\modelseql\P$ by definition of GA censor. It is now straightforward to verify that $\T\cup\C\cup\A'\modelseql\P$, thus contradicting the optimality of $\C$. Consequently, $S(1,\myconstant)$ belongs to every optimal GA of $\E$, i.e.\ $\E \modelsga S(1,\myconstant)$.
    
    Finally, we recall that $\E\modelsga\alpha$ iff $\E\modelsiga\alpha$ when $\alpha$ is a ground atom; thus, both theses are proved.
    \qed
\end{proof}\fi

We then provide an upper bound for the same complexity class. 
We start by showing a fundamental property of $\censiga$. In the following, given a CQE instance $\E=\tup{\T,\P,\A}$, we say that a set of facts from $\clt{\A}$ is \emph{$\E$-disclosable} if it is contained in some GA censor of $\E$, and \emph{$\E$-undisclosable} otherwise. 

\begin{proposition}
\label{pro:iga-min-undisclosable}
    Let $\E=\tup{\T,\P,\A}$ be a CQE instance. For every $\alpha\in\clt{\A}$, $\alpha\in\censiga$ iff $\alpha$ does not belong to any $\subseteq$-minimal $\E$-undisclosable subset of $\clt{\A}$.
\end{proposition}
\iflong\begin{proof}
    $(\Rightarrow)$
    Let $\alpha\notin\censiga$, which implies that $\alpha\notin\C$ for some $\C\in\optcens(\E)$.
    Since $\C$ is optimal, then $\C\cup\{\alpha\}$ is $\E$-undisclosable. Moreover, every $\subseteq$-minimal subset $S$ of $\C$ such that $S$ is $\E$-undisclosable (at least one of which does exist) contains $\alpha$.
    
    $(\Leftarrow)$
    Let $\alpha$ belong to some $\subseteq$ minimal $\E$-undisclosable subset $S$ of $\clt{\A}$.
    By minimality, we have that $S\setminus\{\alpha\}$ is contained in some GA censor $\C$ of $\E$, which obviously does not contain $\alpha$. Consequently, $\alpha\notin\censiga$.
    \qed
\end{proof}\fi

\begin{algorithm}
\caption{$\disclosablealg$}
\label{alg:disclosable}
    \Input{A CQE instance $\E=\tup{\T,\P,\A}$, a set of facts $\C$ such that $\C\subseteq\clt{\A}$;}
    \Output{A Boolean value;}
    \If{there exists a set of facts $\C'$ such that:
        \\\qquad $(i)$ $\C\subseteq\C'\subseteq\clt{\A}$ and
        \\\qquad $(ii)$ for each $\tau\in\P$ and $\sigma\in\gs(\tau,\Const(\A))$ 
        \\\qquad\qquad $\T\cup\C'\not\models\body(\sigma(\tau))$ or
        $\T\cup\C'\models\head(\sigma(\tau))$}{%
        \Return{$\true$;}}
    \Return{$\false$;}
\end{algorithm}

Algorithm~\ref{alg:disclosable} simply guesses a superset of $\C$ and checks that it is a GA censor of $\E$. Then, the next property immediately follows by definition of $\E$-disclosability.

\begin{lemma}
\label{lem:disclosable}
    Let $\E=\tup{\T,\P,\A}$ be a CQE instance and let $\C$ be a set of ground atoms such that $\C\subseteq\clt{\A}$. Then $\disclosablealg(\E,\C)$ returns $\true$ iff $\C$ is an $\E$-disclosable subset of $\clt{\A}$. 
\end{lemma}

%



\begin{algorithm}
\newcounter{cond}
\newcounter{subcond}[cond]
\newcommand{\cond}[1]{%
  \refstepcounter{cond}\label{#1}%
  \ref*{#1}\xspace%
}
\newcommand{\subcond}[1]{%
  \refstepcounter{subcond}\label{#1}%
  \ref*{#1}\xspace%
}
\renewcommand{\thecond}{\textit{(\alph{cond})}}
\renewcommand{\thesubcond}{\textit{(\alph{cond}\arabic{subcond})}}
\caption{$\igaentalg$}
\label{alg:iga-arbitrary-new}
    \Input{A CQE instance $\E=\tup{\T,\P,\A}$, a BUCQ $q$;}
    \Output{A Boolean value;}
    \If{there exist an integer $m\le|\clt{\A}|$ and sets $\A',S_1,\ldots,S_m \subseteq \clt{\A}$ \\\quad (with $\clt{\A}\setminus\A'=\allowbreak\{\alpha_1,\ldots,\alpha_m\}$) such that:
        \\\qquad\cond{cond:a} 
        $\T\cup\A'\not\models q$
        \\\qquad\cond{cond:b}
        for each $1\le i\le m$:
            \\\qquad\quad\subcond{cond:b1} 
            $\alpha_i\in S_i$
            \\\qquad\quad\subcond{cond:b2}
            $\disclosablealg(\E,S_i)=\false$
            \\\qquad\quad\subcond{cond:b3} for each $\beta\in S_i$,
            $\disclosablealg(\E,S_i\setminus\{\beta\})=\true$}{%
        \Return{$\false$;}}
    \Return{$\true$;}
\end{algorithm}

We are now ready to introduce the algorithm $\igaentalg$ (Algorithm~\ref{alg:iga-arbitrary-new}), which allows us to establish the next upper bound.


\begin{theorem}
\label{thm:iga-arbitrary-ub}
    $\IGAEnt[\dlliter,\policyarb]$ is $\pidue$-complete in data complexity.
\end{theorem}
\begin{proof}
    The lower bound has been established in Lemma~\ref{lem:linear+denial-iga-ga-lb}.
    As for the upper bound, we first note that, for $\dlliter$ TBoxes, Algorithm~\ref{alg:disclosable} always terminates. \iflong Moreover, in the given scenario, we prove that the algorithm decides whether $\E\modelsiga q$ in a sound and complete way.\else We now prove that the algorithm is sound and complete for $\dlliter$.
    \fi

    \newcounter{rmrk}
    \smallskip\noindent
    \refstepcounter{rmrk}\label{rem:alg-iga}%
    \textbf{Remark~\arabic{rmrk}.}\;
    Note that, by Lemma~\ref{lem:disclosable}, the combination of Conditions~\ref{cond:b2} and \ref{cond:b3} cause each $S_i$ to be a $\subseteq$-minimal $\E$-undisclosable subset of $\clt{\A}$. Therefore, since \ref{cond:b1} requires $\alpha_i\in S_i$ for every $1\le i\le m$, by Proposition~\ref{pro:iga-min-undisclosable} the whole Condition~\ref{cond:b} holds iff $\A'\supseteq\censiga$.
    
    Now, if $\T\cup\censiga\not\models q$, then the algorithm guesses $\A'=\censiga$ and $\{\alpha_1,\allowbreak\ldots,\allowbreak\alpha_m\}=\clt{\A}\setminus\censiga$, for which both Conditions~\ref{cond:a} and \ref{cond:b} hold (respectively, because $\T\cup\A'=\T\cup\censiga \not\models q$ and by Remark~\ref{rem:alg-iga}).
    This proves the soundness of the algorithm. \iflong
    
    \fi On the other hand, consider any guess for which both Conditions~\ref{cond:a} and \ref{cond:b} hold. By Remark~\ref{rem:alg-iga}, 
    $\A'$ is such that $\A'\supseteq\censiga$. Since $\T\cup\A'\not\models q$, then by monotonicity of FO entailment we have that $\T\cup\censiga\not\models q$, which proves the completeness of the algorithm.

    Finally, note that:
    \iflong
        \begin{enumerate}[label=$(\roman*)$]
            \item 
            computing $\clt{\A}$ requires polynomial time in data complexity;
            \item
            Condition~\ref{cond:a} can be verified in polynomial time in data complexity (by hypothesis);
            \item
            Condition~\ref{cond:b} consists of a polynomial number of calls to an \NP\ oracle (i.e.\ Algorithm~\ref{alg:disclosable}). Indeed, such an algorithm guesses a superset of $S_i$ and then makes a polynomial number of BCQ entailment checks (note that $|\gs(\tau,\Const(\A))| \sim \Theta(n^k)$, where $n=|\A|$ and $k$ is the number of universally quantified variables of $\tau$).
        \end{enumerate}
    \else
        $(i)$ both computing $\clt{\A}$ and checking Condition~\ref{cond:a} can be done in polynomial time in data complexity because $\T$ is expressed in $\dlliter$;
        $(ii)$ Condition~\ref{cond:b} consists of a polynomial number of calls to an \NP\ oracle (i.e.\ Algorithm~\ref{alg:disclosable}). Indeed, such an algorithm guesses a superset of $S_i$ and then makes a polynomial number of BCQ entailment checks (note that $|\gs(\tau,\Const(\A))| \sim \Theta(|\A|^k)$, where $k$ is the number of universally quantified variables of $\tau$).
    \fi
    Then, the thesis follows from the above properties.
    \qed
\end{proof}

The same upper bound can be proven for $\GAEnt$, again using Algorithm~\ref{alg:disclosable}.
\begin{theorem}
\label{thm:ga-arbitrary-ub}
   $\GAEnt[\dlliter,\policyarb]$ is $\pidue$-complete in data complexity.
\end{theorem}
\begin{proof}
    The lower bound follows from Lemma~\ref{lem:linear+denial-iga-ga-lb}. As for the upper bound, observe that $\optcens(\E)$ is the set of all subsets $\C$ of $\clt{\A}$ such that $(i)$ $\C$ is $\E$-disclosable and $(ii)$ for every $\alpha\in\clt{\A}\setminus\C$, $\C\cup\{\alpha\}$ is $\E$-undisclosable.
    Consequently, $\GAEnt$ can be decided by guessing a set $\C\subseteq\clt{\A}$ and checking that $(i)$ $\C$ is $\E$-disclosable, $(ii)$ $\C$ becomes $\E$-undisclosable when adding any $\alpha$ from $\clt{\A}$, and $(iii)$ $\T\cup\C\not\models q$.
    Since checking $\E$-disclosability is an NP task 
    (as shown in the proof of Theorem~\ref{thm:iga-arbitrary-ub}), the thesis follows.
    \qed
\end{proof}

Finally, 
\iflong we note that Example~5 of~\cite{BCLMMRSS25} can be used to show that GA-entailment does not enjoy the indistinguishability property already for policies consisting of denials.
Indeed, that example shows that
\else
as shown in~\cite{MRR25}, we remark that
\fi the set $\censiga$ in general is not a GA censor. When this happens, one can see that there exists no ABox $\A'$ satisfying the requirements of Definition~\ref{def:indistinguishability}.

\begin{proposition}
\label{pro:ga-iga-non-indistinguishable}
    For any DL $\L$, both GA- and IGA-entailment fail to satisfy the indistinguishability property for $\L$ and $\policyarb$.
\end{proposition}
\iflong\begin{proof}
    Recalling Example~1 of~\cite{MRR25}, consider the CQE instance $\E=\tup{\T,\P,\A}$, where $\T=\emptyset$, $\A = \{ C(0),B(1),B(2)\,\}$, and $\P = \{ K(B(1)\land B(2)) \ra K \bot, \forall x\,(K C(x) \ra K\exists y\,B(y)) \}$.
        
    Note that $\censiga$ (i.e.\ $\{C(0)\}$) is not a GA censor of $\E$ (as it does not satisfy the second ED of $\P$).
    Then, consider the BCQ $q=C(0)$ (which is both GA- and IGA-entailed by $\E$), and let $\A'$ be any ABox such that $\T\cup\A'\modelseql\P$ and $\T\cup\A'\modelsga q$ (resp., $\T\cup\A'\modelsiga q$). It is immediate to see that every such $\A'$ is such that $\T\cup\A'\modelsga B(c)$ (resp., $\T\cup\A'\modelsiga B(c)$), for some constant $c$.
    \qed
\end{proof}\fi

\iflong
\subsection{Linear Epistemic Dependencies}
\else
\subsubsection*{Linear Epistemic Dependencies}
\fi
\label{subsec:linear}

All the negative results presented in Section~\ref{sec:ga-iga-arbitrary} motivate the need to investigate different approaches to the CQE problem in order to gain indistinguishability and, possibly, tractability of BUCQ entailment.

We first focus on the notable class of linear EDs. We remark that, as observed in~\cite{MRR25}, for linear EDs, only one optimal GA censor exists for every $\dlliter$ CQE instance $\E$, which obviously coincides with $\censiga$. 
We thus have, in this case, that GA-entailment collapses to IGA-entailment.
Also, the fact that the set $\censiga$ is a GA censor of $\E$ means that it can play the role of the ABox $\A'$ of Definition~\ref{def:indistinguishability}, which implies the next property.

\begin{proposition}
\label{pro:ga-iga-linear-indistinguishable}
    Both GA- and IGA-entailment satisfy the indistinguishability property for $\dlliter$ and $\policylin$.
\end{proposition}

For linear EDs and for $\dlliter$ ontologies, the paper~\cite{MRR25}
showed that when dependencies are also \emph{full} (i.e.\ their heads contain no existentially quantified variable), $\GAEnt$ and $\IGAEnt$ become FO-rewritable. 
We show that the problems remain tractable---more precisely, they are both \PTIME-complete---when the full condition is not applied.
%
%
The aforementioned paper suggests a procedure for computing $\censiga$ 
that, informally, iteratively removes from $\clt{\A}$ any fact that, together with $\T$, leads to a violation of $\P$, until a fixpoint is reached.
Here, we formalize that procedure in the following equivalent way.

\begin{definition}[LPV]
\label{def:lpv}
    Given a linear $\dlliter$ CQE instance $\E=\tup{\T,\P,\A}$ and a subset $\A'$ of $\clt{\A}$, a \emph{linear policy violation (LPV)} of $\E$ w.r.t.\ $\A'$ is an atom $\alpha\in\clt{\A}$ such that
    there exist $\tau\in\P$ and $\sigma\in\gs(\tau)$ such that $\T\cup\{\alpha\}\models\body(\sigma(\tau))$ and $\T\cup\A'\not\models\head(\sigma(\tau))$.
    
    We denote by $\lpvz(\E,\A')$ the set constituted by the union of all the LPVs of $\E$ w.r.t.\ $\A'$.
    Then, we define $\lpv(\E,0)=\emptyset$, and, for every integer $i\ge0$:
    \[
    \lpv(\E,i+1)=\lpvz(\E,\clt{\A}\setminus\lpv(\E,i))
    \]
\end{definition}

We now show the convergence of the above notion of $\lpv(\E,i)$ to a fixpoint.

\begin{proposition}
\label{pro:lpv-convergence}
    Let $\E=\tup{\T,\P,\A}$ be a linear $\dlliter$ CQE instance. Then: $(i)$ for every non-negative integer $i$, $\lpv(\E,i+1)\supseteq\lpv(\E,i)$; $(ii)$ for every integer $n$ such that $n\geq|\clt{\A}|$, $\lpv(\E,n)=\lpv(\E,n+1)$.
\end{proposition}
\iflong\begin{proof}
    First, we prove property $(i)$.
    Base case ($i=0$): Since $\lpv(\E,0)=\emptyset$, the thesis immediately follows.
    Inductive case ($i>0$): Suppose $\lpv(\E,i)\supseteq\lpv(\E,i-1)$ and let $\alpha\in\lpv(\E,i)$. Since $\lpv(\E,i)=\lpvz(\E,\clt{\A}\setminus\lpv(\E,i-1))$, there exist $\tau\in\P$ and $\sigma\in\gs(\tau)$ such that $\T\cup\{\alpha\}\models\body(\sigma(\tau))$ and $\T\cup(\clt{\A}\setminus\lpv(\E,i-1))\not\models\head(\sigma(\tau))$. Now, since $\lpv(\E,i)\supseteq\lpv(\E,i-1)$, it follows that $\T\cup(\clt{\A}\setminus\lpv(\E,i))\not\models\head(\sigma(\tau))$, which implies that $\alpha\in\lpv(\E,i+1)$. Consequently, $\lpv(\E,i+1)\supseteq\lpv(\E,i)$.
    
    As for property $(ii)$: From Definition~\ref{def:lpv} it immediately follows that, if $\lpv(\E,i)=\lpv(\E,i+1)$ for some $i$, then $\lpv(\E,i)=\lpv(\E,j)$ for every integer $j$ such that $j>i$.
    This property and property $(i)$ imply that either $\lpv(\E,|\clt{\A}|)=\clt{\A}$ or there exists $i<|\clt{\A}|$ such that $\lpv(\E,i)=\lpv(\E,i+1)$. In both cases, it follows that, for every $n$ such that $n\geq|\clt{\A}|$, $\lpv(\E,n)=\lpv(\E,n+1)$.
    \qed
\end{proof}\fi

Now, we define $\lpv(\E)$ as $\lpv(\E,|\clt{\A}|)$ (i.e.\ as the fixpoint of $\lpv(\E,i)$).

\begin{proposition}
\label{pro:lpv-iga-censor}
    For every CQE instance $\E=\tup{\T,\P,\A}$, the set $\clt{\A}\setminus\lpv(\E)$ coincides with $\censiga$.
\end{proposition}
\iflong\begin{proof}
    It is easy to see that every $\alpha$ of Definition~\ref{def:lpv} is such that $\{\alpha\}$ is a $\subseteq$-minimal $\E$-undisclosable subset of $\clt{\A}$. Then, the thesis follows by Proposition~\ref{pro:iga-min-undisclosable}.
    \qed
\end{proof}\fi





It is immediate to verify that $\lpv(\E)$ (and therefore, by the above proposition, $\censiga$), can be computed in polynomial time w.r.t.\ the size of $\A$. Then, since standard BUCQ answering over $\dlliter$ ontologies is in \PTIME---in fact, in $\aczero$~\cite{CDLLR07}---, once the set $\censiga$ is computed in polynomial time, one can check again in polynomial time whether $\T\cup\censiga\models q$.
Hence, both $\GAEnt[\dlliter,\policylin]$ and $\IGAEnt[\dlliter,\policylin]$ are in \PTIME\ in data complexity. 
We can also provide a suitable query and set of EDs for which \iflong the problem \fi $\IGAEnt$ is also hard for \PTIME\ (even for empty TBoxes). Thus, the following property holds.

\begin{theorem}
\label{thm:iga-linear-ptime-complete}
    Both ~ $\GAEnt[\dlliter,\policylin]$ ~ and ~ $\IGAEnt[\dlliter,\policylin]$ ~ are \PTIME-complete in data complexity.
\end{theorem}
\iflong\begin{proof}
    \newcommand{\pv}{Vars}
    \newcommand{\bv}{BVars}
    We prove the lower bound by showing a reduction from the \textsc{Horn-Sat} problem.
    
    Let $\phi$ be a set of ground Horn rules, and let us assume w.l.o.g.\ that $\phi$ contains at least one headless rule (i.e.\ a clause with only negated variables). Let $\phi'$ be another set obtained from $\phi$ by slightly modifying rules without heads as follows: every headless rule $r$ is replaced by the rule with the same body as $r$ and with the new variable $u$ in the head; moreover, for every propositional variable $p$ occurring in $\phi'$ (including $u$), we add a rule $p \ra p$ to $\phi'$. Notice that, by construction, $\phi'$ always contains at least two rules with $u$ in the head. It is immediate to verify that $\phi$ is unsatisfiable iff $u$ belongs to all the models (and hence to the minimal model) of $\phi'$.
     
    Now, we associate an identifier $r_i^x$ to every Horn rule in $\phi'$ having the variable $x$ in its head.
    Then, let $h[x]$ be the number of rules of $\phi'$ having the variable $x$ in their head, let $\pv(\phi')$ be the set of propositional variables occurring in $\phi'$ and, for every rule $r_i$, let $\bv(r_i)$ be the set of variables in the body of $r_i$.
    We define the ABox $\A$ as the set of facts:
    \[
    \begin{array}{r@{\,}l}
        \displaystyle
        \bigcup_{r_i\in\phi'} & \{ B(r_{i},x) \mid x \in\bv(r_i) \} \:\cup \\
        \displaystyle
        \bigcup_{x\in\pv(\phi')}\bigcup_{1\le i\le h[x]} & \{  H(r_i^x,x), S(r_i^x,r_j^x) \mid  j=(i \bmod h[x])+1 \}
    \end{array}
    \]
    Moreover, we set $\T=\emptyset$, $q=\exists r\,H(r,u)$ and $\P$ to the following set of linear EDs:
    \[
    \begin{array}{r@{\,}l}
        \forall r,r'\! & (\K S(r,r')\rightarrow \K\exists r'',v \, S(r',r'') \land H(r,v)) \\
        \forall r,v & (\K H(r,v)\rightarrow \K\exists r'\,S(r,r')) \\
        \forall r,v & (\K H(r,v)\rightarrow \K\exists v'\,B(r,v')) \\
        \forall r,v & (\K B(r,v)\rightarrow \K\exists r' \, H(r',v)) \\
    \end{array}
    \]
    Observe that:
    \begin{itemize}
        \item The first two EDs imply that the deletion of any fact of the form $H(r,v)$ causes the deletion of all the facts of the form $H(\_,v)$ and all the $S$-facts related to the EDs having $v$ in their head. 
        \item The third ED implies that, if all the $B$-facts for the body variables of a rule are deleted, then also the $H$-fact for its head variable must be deleted.
        \item The fourth ED implies that, if all the $H$-facts for a variable are deleted, then also all its $B$-fact must be deleted.
    \end{itemize}
    
    The set $\censiga$ of $\E=\tup{\T,\P,\A}$ can be obtained starting from $\clt{\A}$ (i.e.\ from $\A$, because $\T=\emptyset$) and deterministically removing some facts according to the EDs of $\P$.
    We now prove that, for every variable $x\in\pv(\phi')$, no fact of the form $H(\_,x)$ belongs to $\censiga$ (i.e.\ $\E\not\modelsiga q$, which for linear EDs holds iff $\E\not\modelsga q$) iff $x\in I$, where $I$ is the minimal model of $\phi'$.
    
    \smallskip
    $(\Leftarrow)$ 
    Let $x\in I$. Note that this is possible only if $x$ occurs in the head of a rule $r\in\phi'$ that either does not have a body (\emph{unit clause}) or whose body variables all belong to $I$.
    \begin{itemize}
        \item In the first case, due to the first three EDs, $\censiga$ does not contain any atom of the form $H(\_,x)$.
        \item The second case occurs instead when all variables $b_1,\ldots,b_m\in\bv(r)$ either occur in unit clauses themselves or, in turn, occur as head variables in rules whose body variables belong to $I$. By recursively applying the previous and the current point, respectively, one can see that all facts \(
            B(r,b_1),\ldots,B(r,b_m)
        \) are removed from $\censiga$ because of the fourth ED (possibly combined with the first two). 
        As a result, by the first three EDs and since $x$ occurs in the head of $r$, it follows that no fact of the form $H(\_,x)$ belongs to $\censiga$ either.
    \end{itemize}
    
    $(\Rightarrow)$ 
    Let $x$ be such that no fact of the form $H(\_,x)$ belongs to $\censiga$. Note that, since every variable occurs in the head of some rule (by the assumption that a rule $p\ra p$ occurs in $\phi'$ for every $p\in\pv(\phi')$), then there exists a fact $H(\_,x)\in\A$. The absence of such $H$-facts from $\censiga$ can only be due to a violation of the third ED (possibly combined with the first two), i.e.\ all the $B(r,\_)$ facts (for some rule $r=b_1,\ldots,b_m \ra x\in\phi'$) are missing from $\censiga$.
    
    Then, we either have that $r$ is a unit clause of $\phi'$ (which would imply that $x\in I$) or that every fact $B(r,b_i)$ (for $1\le i\le m$) occurring in $\A$ has been removed for building $\censiga$. These last removals can only be due to a violation of the fourth ED (possibly combined with the first two), i.e.\ all facts $H(\_,b_i)$ (for every $1\le i\le m$) are missing from $\censiga$. 
    By recursively applying this point, we conclude that all $b_i$ variables belong to $I$ (implying that also $x\in I$).
    
    By instantiating the above property with $x=u$, we have that $u\in I$ (i.e.\ $\phi$ is unsatisfiable) iff $\E\modelsiga q$.
    \qed
\end{proof}
\fi
\newcommand{\mpv}{\textit{MPV}}
\newcommand{\mpvz}{\textit{MPV}_0}
\newcommand{\inductivecensor}{\textit{IC}}
\newcommand{\approxmpvdlliter}{\textit{AMPV}}
\newcommand{\squad}{\:\:}

\section{The MPV Censor}
\label{sec:mpv}

Towards \iflong the goal of \fi identifying a well-founded, sound approximation of the GA semantics that satisfies indistinguishability in the case of arbitrary EDs, we propose a method to compute a GA censor of $\E$ inspired by the clear, intuitive approach used to build the IGA censor for linear EDs. 
Such a censor is built starting from $\clt{\A}$ and iteratively removing all $\subseteq$-minimal subsets of facts that cause the violation of some ED $\tau\in\P$. 
Such a notion of violation extends the one of LPV 
(Definition~\ref{def:lpv}) by addressing 
an important issue: as long as multiple atoms may appear in the body, it is necessary to perform a minimality check on its images.
To account for this problem, we define the notion of minimal policy violation. 

\begin{definition}[MPV]
\label{def:mpv}
    Given a CQE instance $\E=\tup{\T,\P,\A}$ and a subset $\A'$ of $\clt{\A}$, a \emph{minimal policy violation (MPV)} of $\E$ w.r.t.\ $\A'$ is a $\subseteq$-minimal subset $\A''$ of $\clt{\A}$ such that
    there exist $\tau\in\P$ and $\sigma\in\gs(\tau)$ such that $\T\cup\A''\models\body(\sigma(\tau))$ and $\T\cup\A'\not\models\head(\sigma(\tau))$.
    
    We denote by $\mpvz(\E,\A')$ \iflong the set constituted by \fi the union of all the MPVs of $\E$ w.r.t.\ $\A'$.
    %
    %
    Then, we define $\mpv(\E,0)=\emptyset$, and, for every non-negative integer $i$:
    \[
    \mpv(\E,i+1)=\mpvz(\E,\clt{\A}\setminus\mpv(\E,i)) .
    \]
\end{definition}

The next proposition states a key property of $\mpv(\E,i)$.

\begin{proposition}
\label{pro:mpv-convergence}
    For every CQE instance $\E=\tup{\T,\P,\A}$ and for every integer $n$ such that $n\geq|\clt{\A}|$, $\mpv(\E,n)=\mpv(\E,n+1)$.
\end{proposition}
\iflong\begin{proof}
    From Definition~\ref{def:mpv}, and in a way analogous to the proof of Proposition~\ref{pro:lpv-convergence}, we get the following properties:
    \begin{enumerate}[label=$(\roman*)$]
        \item for every $i$ such that $1\leq i\leq n$, $\mpv(\E,i+1)\supseteq\mpv(\E,i)$;
        \item if $\mpv(\E,i+1)=\mpv(\E,i)$, then $\mpv(\E,j)=\mpv(\E,i)$ for every $j>i$. 
    \end{enumerate}
    
    Now, two cases are possible:
    \begin{enumerate}
        \item there exists an integer $i$ such that $0\leq i\leq |\clt{\A}-1|$ and $\mpv(\E,i+1)=\mpv(\E,i)$. In this case, from the above property $(ii)$ it follows that $\mpv(\E,n)=\mpv(\E,n+1)$ for every integer $n$ such that $n\geq|\clt{\A}|$; 
        \item if $\mpv(\E,i+1)\supset\mpv(\E,i)$ for every integer $i$ such that $0\leq i\leq |\clt{\A}-1|$ (i.e.\ every $\mpv(\E,i+1)$ has at least one more atom than $\mpv(\E,i)$), then $|\mpv(\E,|\clt{\A}|)|\geq|\clt{\A}|$, and since $\mpv(\E,i)$ is by definition a subset of $\clt{\A}$, it follows that 
        $\mpv(\E,|\clt{\A}|)=\clt{\A}$.
        For the same reason, it follows that $$\mpv(\E,|\clt{\A}|+1) \subseteq \clt{\A} \subseteq \mpv(\E,|\clt{\A}|).$$ On the other hand, by the above property $(i)$ it follows that $$\mpv(\E,|\clt{\A}|+1)\supseteq\mpv(\E,|\clt{\A})|.$$ Consequently, $\mpv(\E,|\clt{\A}|+1)=\mpv(\E,|\clt{\A}|)$. Finally, from the above property $(ii)$ it follows that $\mpv(\E,j)=\mpv(\E,|\clt{\A}|)$ for every integer $j$ such that $j\geq |\clt{\A}|$.
    \end{enumerate}
    \qed
\end{proof}\fi

Then, $\mpv(\E)$ denotes the set $\mpv(\E,|\clt{\A}|)$, i.e.\ it is the fixpoint of $\mpv(\E,k)$.
We now analyze how such a set relates to the notion of GA censor.

\begin{proposition}
\label{pro:mpv-censor}
    For every CQE instance $\E=\tup{\T,\P,\A}$:
    \begin{itemize}
        \item[$(i)$] for every set of ground atoms $\A'$, if $\A'\supseteq\mpvz(\E,\clt{\A}\setminus\A')$, then 
        $\clt{\A}\setminus\A'$ is a GA censor of $\E$;
        \item[$(ii)$] the set $\clt{\A}\setminus\mpv(\E)$ is a GA censor of $\E$;
        \item[$(iii)$] $\mpv(\E)$ is the smallest set $\A'\subseteq\clt{\A}$ \iflong such that \else with \fi $\A'\supseteq\mpvz(\E,\clt{\A}\setminus\A')$.
    \end{itemize}
\end{proposition}
\iflong\begin{proof}
    Let $\C$ denote the set $\clt{\A}\setminus\A'$. Obviously, $\C\cap\A'=\emptyset$.
    By contradiction, let $\T\cup\C\not\modelseql\P$. 
    Now, let $\A''$ be the $\subseteq$-minimal subset of $\C$ such that, for some $\tau\in\P$ and $\sigma\in\gs(\tau)$, $\T\cup\A''\models\body(\sigma(\tau))$ and $\T\cup\C\not\models\head(\sigma(\tau))$ (notice that such $\A''$ exists because $\T\cup\C\not\modelseql\P$).
    Since $\C\subseteq\clt{\A}$ (by construction), then $\A''$ is also a $\subseteq$-minimal subset of $\clt{\A}$ such that, for some $\tau\in\P$ and $\sigma\in\gs(\tau)$, $\T\cup\A''\models\body(\sigma(\tau))$ and $\T\cup\C\not\models\head(\sigma(\tau))$, i.e.\ $\A''\subseteq\mpvz(\E,\C)$. However, since $\mpvz(\E,\C)\subseteq\A'$ (by hypothesis), this contradicts the fact that $\C\cap\A'=\emptyset$, thus proving property $(i)$.
    
    \smallskip
    Then, by definition of $\mpv(\E)$, we have $\mpv(\E)=\mpvz(\E,\clt{\A}\setminus\mpv(\E))$, hence by property $(i)$ we get property $(ii)$.
    
    \smallskip
    Finally, let $\A'$ be any set of facts such that $\A'\supseteq\mpvz(\E,\clt{\A}\setminus\A')$. We prove by induction that, for every $i$ such that $1\leq i\leq |\clt{\A}|$, $\A'\supset\mpv(\E,i)$.
    \begin{itemize}
        \item Base case ($i=0$): trivially, $\A'\supseteq\emptyset=\mpv(\E,0)$.
        \item Inductive case: suppose $\A'\supseteq\mpv(\E,i)$. Notice that, for every pair of sets $S,S'$, the set $\mpvz(\E,S)$ monotonically decreases when $S$ increases, meaning that $\mpvz(\E,S)\supseteq\mpvz(\E,S')$ whenever $S\subseteq S'$. This holds because, by Definition~\ref{def:mpv}, the MPVs are searched in both cases among all possible subsets of $\clt{\A}$, but the non-entailment check is done on $S$ and $S'$, respectively. Therefore, $\mpvz(\E,\clt{\A}\setminus\A')\supseteq\mpvz(\E,\clt{\A}\setminus\mpv(\E,i))$, and since $\mpv(\E,i+1)=\mpvz(\E,\clt{\A}\setminus\mpv(\E,i))$, it follows that $\A'\supseteq\mpv(\E,i+1)$.
    \end{itemize}
    Since $\mpv(\E)=\mpv(\E,|\clt{\A}|)$, it follows that $\A'\supseteq\mpv(\E)$, thus proving property $(iii)$.
    \qed
\end{proof}\fi

The above property leads naturally to the definition of a new notion of censor.

\begin{definition}[MPV Censor]
\label{def:mpv-censor}
    Given a CQE instance $\E=\tup{\T,\P,\A}$, we define the \emph{MPV censor of $\E$}, denoted by $\corecens$, as the set $\clt{\A}\setminus\mpv(\E)$.
\end{definition}

By Proposition~\ref{pro:mpv-convergence}, for every CQE instance $\E$, the MPV censor of $\E$ always exists and is unique.

\begin{example}
\label{ex:mpv}
    Consider the sets $\T$, $\P$, and $\A$ of Example~\ref{ex:intro}, and let $\E=\tup{\T,\P,\A}$. Note that $\{\parentalCons(\bob,\ann)\}$, $\{\parentOf(\bob,\ann)\}$, and $\{\minor(\ann),\allowbreak\affectedBy(\ann,\disConst)\}$ are all MPVs of $\E$ w.r.t.\ $\clt{\A}$, and that the fixpoint of $\mpv(\E,i)$ is already reached for $i=1$. Then, $\mpvcens=\clt{\A}\setminus\mpv(\E,1)=\{\disease(\disConst)\}$, which is contained in the (unique) optimal GA censor mentioned in Example~\ref{ex:intro}.
    \qedex
\end{example}

\begin{example}
    Consider again the CQE instance of Example~\ref{ex:intro} and let $\A$ also contain the facts $\treatedBy(\ann,\cloe)$ and $\pediatrician(\cloe)$, \iflong modelling the fact \else meaning \fi that $\ann$ is treated by ($\treatedBy$) $\cloe$, who is a pediatrician ($\pediatrician$).
    Moreover, let $\P$ also contain the additional ED $\tau = \forall x,y\,(\K(\treatedBy(x,y) \land \pediatrician(y)) \ra \K\minor(x))$, which allows to reveal that an individual $x$ is treated by a pediatrician only if it can also be revealed that $x$ is a minor.
    In this case, 
    the fixpoint of $\mpv(\E,i)$ is reached for $i=2$. Specifically, 
    $\mpv(\E,1)$ is as before, while 
    $\mpv(\E,2)=\mpv(\E,1)\cup\{\treatedBy(\ann,\cloe),\allowbreak\pediatrician(\cloe)\}$.
    Thus, $\mpvcens=\{\disease(\disConst)\}$. 
    \qedex
\end{example}




We now analyze the relation between the MPV \iflong censor and the \else and \fi IGA censor.
We start by showing the next property, \iflong whose proof is based on the following intuition: \else based on the intuition that \fi if there exists any overlap between a $\subseteq$-minimal $\E$ undisclosable set and a set $\C$ obtained by iteratively removing MPVs, then $\C$ has not reached the fixpoint yet (i.e.\ $\C\ne\corecens$).

\begin{theorem}
\label{thm:sound-approx}
    Let $\E$ be a CQE instance. Then, $\corecens\subseteq\censiga$.
\end{theorem}
\begin{proof}
	\newcommand{\mus}{\mathit{MUS}}
    \newcounter{claim}
	In what follows, let $\mus(\E)$ be the set of all $\subseteq$-minimal $\E$-undisclosable subsets of $\clt{\A}$.

    \smallskip\noindent
    \refstepcounter{claim}\label{cl:mus-intersection}%
    \textbf{Claim~\arabic{claim}.}\; \emph{Let $\C$ be a GA censor of $\E$ such that there exists a $S\in\mus(\E)$ for which $\C\cap S\ne\emptyset$. Then, there exist $\alpha\in\C$, $\tau\in\P$, $\sigma\in\gs(\tau)$, and $I\in\im(\body(\sigma(\tau)),\clt{\A},\T)$ such that $\alpha\in I$ and $\T\cup \C\not\models\head(\sigma(\tau))$.}

    \renewcommand{\SS}{\mathbb{S}}
    We now prove the above claim.
    Let $\SS$ be the set of all sets in $\mus(\E)$ 
    intersecting $\C$ (which, by hypothesis, is non-empty). Note that no set in $\SS$ is a singleton (otherwise it would be contained in $\C$, hence contradicting the fact that $\C$ is a GA censor of $\E$).

    Let now $S_1\in\mus(\E)$ and $\alpha_1\in S_1\cap\C$, and suppose first that $(S_1\cup\C)\setminus\{\alpha_1\}$ is $\E$-disclosable, i.e.\ it is contained in some GA censor $\C_1$ of $\E$. Since $S_1\subseteq\C_1\cup\{\alpha_1\}$, then $\C_1\cup\{\alpha_1\}$ is not a GA censor of $\E$, i.e.\ $\C_1\cup\{\alpha_1\}\not\modelseql\P$. Then, for some $\tau\in\P$ and $\sigma\in\gs(\tau)$, $\T\cup\C_1\cup\{\alpha_1\}\models\body(\sigma(\tau))$ and $\T\cup\C_1\cup\{\alpha_1\}\not\models \head(\sigma(\tau))$.
    Now, let $I$ by any $\subseteq$-minimal subset of $\C_1\cup\{\alpha_1\}$ such that $\T\cup\I\models\body(\sigma(\tau))$. Note that $I$ cannot be included in $\C_1$ (as it would contradict the fact that $\T\cup\C_1\modelseql\tau$), which implies that $\alpha_1\in I$. 
    Furthermore, by monotonicity, $I\in\im(\body(\sigma(\tau)),\clt{\A},\T)$ and $\T\cup\C\not\models\head(\sigma(\tau))$. Thus, in this case, the thesis would follow.

    On the other hand, let $(S_1\cup\C)\setminus\{\alpha_1\}$ be $\E$-undisclosable. Then, it contains a set $S_2\in\mus(\E)$. 
    Note that $\alpha_1\not\in S_2$, and that $S_2$ cannot be entirely contained in $\C$ (because $\C$ is a GA censor of $\E$) nor in $S_1$ (by minimality of $S_1$): $S_2$ is therefore only partially contained in $\C$, i.e.\ $S_2\in\SS$.

    Similarly, let $\alpha_2\in S_2\cap\C$, and suppose first that $(S_2\cup\C)\setminus\{\alpha_1,\alpha_2\}$ is $\E$-disclosable, i.e.\ it belongs to some GA censor $\C_2$ of $\E$. Again, $\C_2\cup\{\alpha_2\}$ is not a GA censor of $\E$ (as it entirely contains $S_2$), which (analogously to the above) leads to concluding that the thesis holds.
    On the other hand, supposing that $(S_2\cup\C)\setminus\{\alpha_1,\alpha_2\}$ is $\E$-undisclosable would imply that it contains a set $S_3\in\SS$ that, like before, does not contain $\alpha_1$ and $\alpha_2$.

    Since the set $\SS$ is finite, the claim follows by iterating the above argument.

    Suppose now that there exists 
    $S\in\mus(\E)$ with $\corecens\cap S\ne\emptyset$. Since, by Proposition~\ref{pro:mpv-censor}, $\censmpv$ is a GA censor of $\E$, then by Claim~\ref{cl:mus-intersection} there exists a fact $\alpha\in\censmpv$ belonging to an MPV of $\E$ w.r.t.\ $\censmpv$\iflong.
    This would, however, contradict \else, which contradicts \fi the definition of MPV censor.
	Thus, $\corecens\subseteq \clt{\A}\setminus\bigcup_{S\in\mus(\E)} S = \censiga$.
    \qed
\end{proof}

Then, we show the following strict correspondences between $\corecens$ and $\censiga$.

\iflong
\begin{lemma}
\label{lem:mpv-iga-denials}
    Let $\E=\tup{\T,\P,\A}$ be a CQE instance such that $\mpv(\E)=\mpv(\E,1)$. Then, $\corecens=\censiga$.    
\end{lemma}
\begin{proof}
    Note that, if $\mpv(\E)=\mpv(\E,1)$, then $\mpv(\E)$ coincides with the sets of $\subseteq$-minimal $\E$-undisclosable subsets of $\clt{\A}$. Then, the thesis follows by Definition~\ref{def:mpv-censor} and Proposition~\ref{pro:iga-min-undisclosable}.
    \qed
\end{proof}
\fi

\begin{theorem}
\label{thm:mpv-iga-denials}
\label{thm:mpv-iga-linear-dlliter}
    Let $\E=\tup{\T,\P,\A}$ be a CQE instance. Then: $(i)$ if $\P$ is a set of denials, then, $\corecens=\censiga$; $(ii)$ if $\T$ is a $\dlliter$ TBox and $\P$ is linear, then $\corecens=\censiga$.
\end{theorem}
\iflong\begin{proof}
    Note that, in the hypothesis $(i)$, the fixpoint for $\mpv(\E,i)$ is already reached for $i=1$ (i.e.\ $\mpv(\E,1)=\mpv(\E)$).
    Then, $\mpv(\E)$ coincides with the sets of $\subseteq$-minimal $\E$-undisclosable subsets of $\clt{\A}$. Hence, the thesis follows by Definition~\ref{def:mpv-censor} and Proposition~\ref{pro:iga-min-undisclosable}.
    
    Under the hypothesis $(ii)$, Definition~\ref{def:lpv} and Definition~\ref{def:mpv} coincide; thus, the thesis follows immediately.
    \qed
\end{proof}\fi

In the remainder of this section, we study BUCQ entailment under this new notion of censor, for $\dlliter$ 
CQE instances. Given a BUCQ $q$, we write $\E\modelscore q$ if $\T\cup\corecens\models q$.
Note that, by Theorem~\ref{thm:sound-approx}, this entailment semantics is a sound approximation of IGA-entailment, which in turn soundly approximates GA-entailment. 
Then, we define the decision problem related to MPV-entailment, again parameterized with respect to a DL $\L$ and a class of policies $\policyclass$. Given an $\L$ CQE instance $\E=\tup{\T,\P,\A}$ such that $\P\in\policyclass$ and a BUCQ $q$, $\CoreGAEnt[\L,\policyclass]$ is the problem of checking whether $\E\modelscore q$.

\begin{example}
\label{ex:mpv-ent}
    Consider again the CQE instance $\E$ of Example~\ref{ex:intro}. Then, it is easy to see that $\E\modelsmpv \exists x\,\disease(x)$ and $\E\not\modelsmpv\exists x\,\affectedBy(x,\disConst)$. On the other hand, one can verify that $\E\modelsiga\exists x\,\affectedBy(x,\disConst)$.
    \qedex
\end{example}

\begin{theorem}
\label{thm:mpv-ent-dlliter}
    $\CoreGAEnt[\dlliter,\policylin]$ and $\CoreGAEnt[\dlliter,\policyarb]$ are, respectively, \PTIME-hard and in \PTIME\ in data complexity.
\end{theorem}
\begin{proof}
    The \PTIME-hardness follows immediately from Theorem~\ref{thm:iga-linear-ptime-complete} and property $(ii)$ of Theorem~\ref{thm:mpv-iga-linear-dlliter}. \iflong
    
    \fi To prove membership in \PTIME, we show that, for every integer $i$ such that $1\leq i\leq |\clt{\A}|$, the set $\mpv(\E,k)$ can be computed in polynomial time. This is a consequence of the following properties:
    \iflong
    \begin{itemize}
        \item standard BUCQ answering for $\dlliter$ is in \PTIME,
        hence $\clt{\A}$ can be computed in polynomial time w.r.t.\ the size of $\A$;
        \item given any ED $\tau$, there are polynomially many substitutions in $\gs(\tau,\Const(\A))$;
        \item given any ED $\tau$ and substitution $\sigma$, there are polynomially many images of $\body(\sigma(\tau))$ in $\clt{\A}$ w.r.t.\ $\T$, because in every $\dlliter$ ontology, every image of a BCQ $q'$ has size at most $k$, where $k$ is the number of atoms in $q'$~\cite{CDLLR07}.
    \end{itemize}
    \else
        $(i)$ standard BUCQ answering for $\dlliter$ is in \PTIME,
        hence $\clt{\A}$ can be computed in polynomial time w.r.t.\ the size of $\A$;
        $(ii)$ given any ED $\tau$, there are polynomially many substitutions in $\gs(\tau,\Const(\A))$;
        $(iii)$ given any ED $\tau$ and substitution $\sigma$, there are polynomially many images of $\body(\sigma(\tau))$ in $\clt{\A}$ w.r.t.\ $\T$, because in every $\dlliter$ ontology, every image of a BCQ $q'$ has size at most $k$, where $k$ is the number of atoms in $q'$~\cite{CDLLR07}.
    \fi
    Therefore, every set $\mpv(\E,i)$ can be computed in polynomial time w.r.t.\ the size of $\A$, which implies the thesis.
    \qed
\end{proof}

We finally state the following fundamental property, which derives from the fact that the MPV censor can always be used as the ABox $\A'$ of Definition~\ref{def:indistinguishability}.

\begin{proposition}
\label{pro:mpv-indistinguishable}
    For every DL $\L$, MPV-entailment satisfies the indistinguishability property for $\L$ and $\policyarb$.
\end{proposition}
    \section{Implementation and Experiments}
\label{sec:experiments}

\newcommand{\benchFive}{\textsf{o2b}\textsubscript{5}\xspace}
\newcommand{\benchTen}{\textsf{o2b}\textsubscript{10}\xspace}
\newcommand{\hi}{\textit{HI}}
\newcommand{\ucqrewr}{\textit{QRew}}

We have conducted an experimental evaluation of MPV-entailment over $\dlliter$ CQE instances.
Our experiments are based on an implementation of an algorithm for MPV-entailment that 
works with $\dlliter$ CQE instances (Algorithm~\ref{alg:mpv-ent-dlliter}). Before describing the algorithm, we introduce some auxiliary notions.

Given a BCQ $q$, a homomorphism for $q$ is a mapping $h$ of the terms occurring in $q$ to individuals, such that $h(a)=a$ for every individual $a$. 
Given a homomorphism $h:\vseq{x}\rightarrow\vseq{t}$ for $q$, we denote by $h(q)$ the set of atoms 
$
\{ p(\vseq{t}) \mid p(\vseq{x}) \textrm{ occurs in } q \textrm{ and } h(\vseq{x})=\vseq{t} \}.
$

Given an ABox $\A$, we call a subset $\A'$ of $\A$ a \emph{homomorphic image} of $q$ in $\A$ if there exists a homomorphism $h$ for $q$ such that $\A'= h(q)$.
We denote by $\hi(q)$ the set of all homomorphic images of $q$.

\begin{algorithm}[t]
    \caption{$\mpventdlliteralg$}
    \label{alg:mpv-ent-dlliter}
    \Input{A $\dlliter$ CQE instance $\tup{\T,\P,\A}$, a BUCQ $q$;}
    \Output{A Boolean value;}
    $\A'\gets\emptyset$;\\    
    \Repeat{$\A'=\A_p'$}{
        $\A_p'\gets\A'$;\\        
        $\I\gets\emptyset$;\\        
        \ForEach{$\tau\in\P$}{
            \ForEach{$\sigma\in\gs(\tau,\Const(\A))$}
            {
                \If{%
                    \label{step:ed-violation}
                    $\T\cup\A\models\body(\sigma(\tau))$ \textrm{and} 
                    $\T\cup(\clt{\A}\setminus\A_p')\not\models\head(\sigma(\tau))$
                }{%
                    $\I\gets\I\cup\{ \A'' \mid
                    \A''\subseteq\clt{\A}$, 
                    $\A''\in\hi(\sigma(q'))$ \textrm{and} 
                    $q'\in\ucqrewr(\body(\tau),\T) \}$;
                }
            }
        }
        $\I\gets\{S\in\I\mid \nexists S'\in\I \mbox{ s.t.\ }S'\subset S \}$;\\
        \label{step:minimality-check}
        $\A'\gets\A'\cup\bigcup_{\A''\in\I}\A''$;
    }
    \If{%
        $\T\cup\clt{\A}\setminus\A'\models q$
    }{%
        \Return{$\true$;}
    }
    \Return{$\false$;}
\end{algorithm}

We are now ready to present the algorithm $\mpventdlliteralg$ (Algorithm~\ref{alg:mpv-ent-dlliter}) that decides MPV-entailment of BUCQs over $\dlliter$ CQE instances.
In the algorithm, $\ucqrewr$ denotes a procedure for UCQ rewriting in $\dlliter$: for every $\dlliter$ TBox $\T$ and UCQ $q$, $\ucqrewr(q,\T)$ returns a UCQ $q'$ such that, for every ABox $\A$, $\T\cup\A\models q(\vseq{c})$ iff $\A\models q'(\vseq{c})$.

\begin{proposition}
\label{pro:alg-mpv-ent-dlliter}
    Let $\E$ be a $\dlliter$ CQE instance and let $q$ be a BUCQ. Then, $\mpventdlliteralg(\E,q)$ returns \true\ iff $\E\modelsmpv q$. 
\end{proposition}
\iflong\begin{proof}
    First, recall that, for every $\tau\in\P$, $\ucqrewr(\body(\tau),\T)$ returns a UCQ $q'$ such that for every ABox $\A$ and for every $\sigma\in\gs(\tau,\Const(\A))$, $\T\cup\A\models\body(\sigma(\tau))$ iff $\A\models\sigma(q')$.
    
    Then, we prove the following property (*): at every iteration $i$ of the repeat--until loop of the algorithm, the set of sets $\I$ computed after the minimization step of line~\ref{step:minimality-check} is the set of all the minimal policy violations of $\E$ with respect to $\A_p'$.
    %
    Suppose $\A''$ is a minimal policy violation of $\E$ with respect to $\A_p'$. Then, there exist $\tau\in\P$ and $\sigma\in\gs(\tau,\Const(\A))$ such that $\T\cup\A''\models\body(\sigma(\tau))$ and $\T\cup(\C\setminus\A_p')\not\models\head(\sigma(\tau))$. Consequently, $\A''\models\sigma(q')$, where $q'$ is the UCQ returned by $\ucqrewr(\body(\tau),\T)$. Consequently, $\A''$ is a homomorphic image of $\ucqrewr(\body(\tau),\T)$ in $\clt{\A}$, hence $\A''\in\I$ before the minimization of $\I$.
    Furthermore, it is immediate to verify that every set $\A''$ that belongs to the set of sets $\I$ before its minimization contains a minimal policy violation of $\E$ with respect to $\A_p'$. These two properties imply that, before its minimization, the set of sets $\I$ computed by the algorithm contains all the minimal policy violations of $\E$ w.r.t.\ $\A_p'$ and supersets of them. Consequently, property (*) follows.
    
    Finally, property (*) immediately implies that the set $\A'$ computed by the algorithm at the $i$-th iteration corresponds to $\mpv(\E,i)$. Consequently, the set $\C$ computed by the algorithm corresponds to $\corecens$, which implies the thesis.
    \qed
\end{proof}\fi

We implemented the algorithm $\mpventdlliteralg$ as a procedure operating over a relational database storing ABox facts.
We ran the algorithm both in its full version and in a modified version (called $\ampventdlliteralg$), which does not perform the elimination of non-minimal homomorphic images of the rewriting of the bodies of the violated EDs (i.e.\ it skips Step~\ref{step:minimality-check} of Algorithm~\ref{alg:mpv-ent-dlliter}). In this way, $\ampventdlliteralg$ computes a set of atoms $\A'$ that is in general larger than $\mpv(\E,i)$, since also non-minimal homomorphic images (i.e.\ non-minimal policy violations) are deleted from $\clt{\A}$. Such a set $\A'$ is still a GA censor (which is an immediate consequence of property $(i)$ of Proposition~\ref{pro:mpv-censor}), and we call it the AMPV censor of $\E$. Consequently, $\ampventdlliteralg$ computes a sound approximation of MPV-entailment. Nevertheless, our experiments showed that such a minimization step is often computationally expensive, so we decided to test its real impact in terms of query answers.

Experiments were run on a laptop with an Intel Core i7-7700HQ processor running at 2.80~GHz, and 8~GB of RAM; the prototype is written in Java~11 and relies on MySQL~8.0 for ABox storage and manipulation.
The employed OWL ontology belongs to the OWL2Bench~\cite{SBM20} benchmark, and it consists of a fixed TBox about the university domain, 10 OWL~2~QL queries~\cite{owl2bench-queries}, and a tool that, given as input a positive integer $N$ (representing the number of universities), generates an ABox of proportional size. We considered $N\in\{1,5,10,20,25,50\}$, using the default seed for ABox generation.
Furthermore, we defined a policy consisting of 11 EDs of the general type.
\iflong The source code of our implementation is publicly available at \url{https://github.com/iswc-2026-anon/MPVCensor}.\fi

All the data manipulation steps required by Algorithm~\ref{alg:mpv-ent-dlliter} have been implemented, in practice, using suitable SQL \texttt{UPDATE} and \texttt{DELETE} operations.
In our implementation, the function $\ucqrewr$ is provided by the tree-witness query rewriter for OWL2~QL ontologies\footnote{\url{https://titan.dcs.bbk.ac.uk/~roman/tw-rewriting/}}~\cite{KKZ12,MKZ13}.
More details about the implementation and optimizations are provided in the Appendix.


\begin{table}[t]
    \centering
    \begin{tabular}{c|c|c|c|c|c|c|}
        \cline{2-7}
        \multicolumn{1}{c|}{\multirow{2}{*}{}} &
        \multicolumn{6}{c|}{Dataset size ($N$)} \\
        \cline{2-7}
         & 1 & 5 & 10 & 20 & 25 & 50 \\
        \hline
        \multicolumn{1}{|c|}{$|\A|$} & 50.2k & 325.5k & 711.1k & 1.4M & 1.7M & 3.5M \\
        \multicolumn{1}{|c|}{$|\clt{\A}|$} & 89.6k & 595.1k & 1.3M & 2.5M & 3.1M & 6.4M \\
        \multicolumn{1}{|c|}{$\#_{\textrm{ampv}}$} & 14.6k & 91.1k & 191.7k & 383.9k & 462.3k & 960.1k \\
        \multicolumn{1}{|c|}{$\#_{\textrm{mpv\phantom{a}}}$} & 13.3k & 82.5k & 173.3k & -- & -- & -- \\
        \multicolumn{1}{|c|}{$t_{\textrm{ampv}}$}
        & 29.1 & 84.3 & 136.7 & 295.1 & 305.7 & 689.1 \\
        \multicolumn{1}{|c|}{$t_{\textrm{mpv\phantom{a}}}$}
        & 80.7 & 2568.3 & 10621.6 & t.o. & t.o. & t.o. \\
        \hline
    \end{tabular}
    
    \caption{This table reports the cardinality of the original ABox $\A$ and of its closure w.r.t.\ $\T$, other than the number ($\#$) of facts removed from $|\clt{\A}|$ and the time ($t$, expressed in seconds) to compute the (A)MPV censor.}
\label{tab:core-ga-censor-summary}
\end{table}


\begin{table}[t]
    \centering
    \newcommand{\evalTime}{$t_a$\xspace}
\newcommand{\evalTimeMPV}{$t_m$\xspace}
\newcommand{\numAnsAMPV}{\#$_a$\xspace}
\newcommand{\numAnsMPV}{\#$_m$\xspace}
\newcommand{\numAnsEmptyPol}{\#$_\emptyset$\xspace}

\begin{tabular}{|c|l|c|c|c|c|c|c|c|c|c|}
\hline
\multicolumn{1}{|c}{\multirow{2}{*}{N}} &
\multicolumn{1}{c|}{\multirow{2}{*}{}} &
\multicolumn{9}{c|}{Benchmark query} \\
\cline{3-11}
\multicolumn{2}{|c|}{} & $q_1$ & $q_2$ & $q_3$ & $q_5$ & $q_6$ & $q_7$ & $q_8$ & $q_9$ & $q_{10}$\\
\hline

\multirow{5}{*}{$1$}
& \evalTime &2.5  &11.1  &0.8  &1.1 &0.8 &13.2  &4.5  &12.5 &3.2 \\
& \evalTimeMPV &2.7  &10.2  &0.6  &0.7 &0.8  &53.3  &6.9  &12.7 &152.9 \\
& \numAnsAMPV &686  &1684  &2  &30 &0 &579  &693  &145 &71 \\
& \numAnsMPV &943  &1684  &2  &30 &0  &579  &1701  &145  &71\\
& \numAnsEmptyPol &1427  &1684  &6  &666  &2422 &858  &2825 &145 &106  \\
\hline

\multirow{5}{*}{$5$}
& \evalTime &6.9  &89.9  &0.7  &1.1 &0.5 &76.2  &18.5  &58.5 &25.1  \\
& \evalTimeMPV &9.4  &97.5  &0.4  &1.3 &0.5  &50.2  &38.2  &110.3 &68.6 \\
& \numAnsAMPV &4626  &18872  &5  &202 &0 &3951  &4418  &1698 &453  \\
& \numAnsMPV &6330  &18872  &5  &202 &0 &3951  &11286  &1698  &453\\
& \numAnsEmptyPol &9228  &18872  &34  &3574 &16236 &5489  &17904 &1698 &642  \\
\hline

\multirow{5}{*}{$10$}
& \evalTime &15.5  &347.9  &0.4  &1.8 &0.4 &297.8  &36.6  &108.4 &40.4 \\
& \evalTimeMPV &32.1  &1172.9  &0.4  &1.6 &0.4 &201.9  &93.6  &101.1 &138.5 \\
& \numAnsAMPV &9947  &44190  &14  &424 &0 &9132  &9838  &3434 &1109  \\
& \numAnsMPV &13542  &44190  &14  &424 &0 &9132  &24650  &3434 &1109  \\
& \numAnsEmptyPol &19782  &44190  &75  &6564 &35889 &11969  &39278 &3434 &1413  \\
\hline

\multirow{3}{*}{$20$}
& \evalTime &29.5  &1342.6  &0.6  &3.2 &0.5  &430.6  &63.3  &259.2 &66.7  \\
& \numAnsAMPV &19109  &88268  &22  &920 &0 &16803  &19151  &6842 &2193  \\
& \numAnsEmptyPol &38251  &88268  &143  &14962 &69397 &23124  &76555 &6842 &2750  \\
\hline

\multirow{3}{*}{$25$}
& \evalTime &37.4  &1848.1  &0.4  &5.1 &0.4 &541.8  &90.4  &297.9 &106.5 \\
& \numAnsAMPV &23586  &109957  &30  &926 &0 &21755  &23264  &8400 &2684 \\
& \numAnsEmptyPol &47336  &109957  &180  &15788 &86376 &28637  &94381  &8400 &3444 \\
\hline

\multirow{3}{*}{$50$}
& \evalTime &263.1  &4058.1  &0.5  &7.1 &0.4 &1051.1  &245.1  &690.9  &345.6 \\
& \numAnsAMPV &48552  &227120  &64  &2316 &0 &43098  &47538  &18288  &5052 \\
& \numAnsEmptyPol &96486  &227120  &365  &36588 &174947 &58409  &192944 &18288 &7137 \\
\hline
\end{tabular}

    \caption{For each query and dataset size, the table reports the execution time (\evalTime, in m
    s) and the number of answers (\numAnsAMPV) for queries posed to the AMPV censor, while \evalTimeMPV and \numAnsMPV refer to query evaluation over the MPV censor. We also report the query answers obtained in case $\P=\emptyset$ (\numAnsEmptyPol) for a comparison.
    }
\label{tab:qa-summary-main}
\label{tab:qa-summary}
\end{table}


Tables~\ref{tab:core-ga-censor-summary} and \ref{tab:qa-summary-main} \iflong
report the 
performance evaluation of our experiments
\else summarize our experimental evaluation
\fi, in particular regarding the computation of the (A)MPV censor and the subsequent query evaluation. 

Considering Table~\ref{tab:core-ga-censor-summary}, we remark that, differently from the MPV censor, the time required for computing the AMPV censor increases linearly w.r.t.\ the size of the original dataset.
On the other hand, the minimality check of Algorithm~\ref{alg:mpv-ent-dlliter} crucially affects the overall time required for computing the MPV censor.

As for query evaluation, we considered all 10 CQs of the OWL2Bench benchmark; however, one of these queries ($q_4$) produced no answers even in the uncensored case, so we do not report results about it. Then, Table~\ref{tab:qa-summary-main} shows that:
$(i)$ all queries, except for $q_2$ and $q_9$, are affected significantly by the data protection policy;
$(ii)$ one query ($q_6$) does not have any answer in the MPV censor;
$(iii)$ 
for two queries ($q_1$ and $q_8$), 
\iflong the number of answers in the AMPV censor is significantly smaller than in the MPV censor.
\else 
the AMPV censor returns significantly fewer answers than the MPV censor.
\fi
All the other queries yield the same number of answers.




    \section{Conclusions}
\label{sec:conclusions}

We investigated CQE under policies expressed via epistemic dependencies (EDs), focusing on the use of ground atom (GA) censors for safe information disclosure. We first investigated GA- and IGA-entailment
and analyzed their data complexity in the presence of ED-based policies when the TBox is expressed in $\dlliter$, both in the case of linear ED
and in the general case.
Since both semantics resulted in being computationally hard in terms of data complexity, we defined a new semantics, called MPV, and showed that it is computationally easier and provides a sound approximation of query answering with respect to the previous semantics. 
Finally, we evaluated our approach with a prototype implementation, using the OWL2Bench benchmark.

The present contribution can be extended in different directions. Potential future work includes:
    finding new sufficient conditions for 
    the MPV censor to coincide with the intersection $\censiga$ of all optimal GA censors;
    studying the practical impact of the MPV semantics on the amount of information missed w.r.t.\ the IGA semantics;
    extending the computational study of MPV-entailment to ontologies expressed in other lightweight DLs like $\EL$ or $\RL$, the logics underlying the OWL 2 EL and OWL 2 RL profiles;
    finding alternative, well-founded semantics that soundly approximate $\GAEnt$ and preserve tractability of query answering 
    even when using arbitrary EDs.

    \iflong\else
\paragraph*{Supplemental Material Statement:}
Complete proofs of theorems and lemmas and further details about implementation and experiments \ifcameraready
are included in the full version of this paper available on arXiv~\cite{MRR26_arXiv}
\else
are attached to the submission in EasyChair; if the paper gets accepted, they will be included in an extended version to be published on arXiv.
\fi
\ifcameraready
Our system's source code can be downloaded from \url{https://github.com/iswc-2026-anon/MPVCensor}.
The benchmark we used is available at \url{https://github.com/kracr/owl2bench}.
\else
The attachment 
also contains a documented, ready-to-use software (whose source code is available at \url{https://github.com/iswc-2026-anon/MPVCensor}) for replicating our experiments. The benchmark we used is available at \url{https://github.com/kracr/owl2bench}.

\fi

        \paragraph{Declaration of use of Generative AI:} During the preparation of this work, the authors used ChatGPT-5.2 for grammar and spelling checks. After using this tool, the authors reviewed and edited the content as needed, and they take full responsibility for the publication’s content.

        \ifcameraready
            \begin{credits}
    \subsubsection{\ackname} 
    This work was partially supported by the EU under the HORIZON.2.1.5 project dAIbetes (grant id.~101136305).
\end{credits}
        \fi
    \fi
    
    \bibliographystyle{splncs04}
    \bibliography{bibliography/strings,bibliography/bibliography,bibliography/w3c}

    \newpage
    \appendix
\section*{Appendix}
\longtrue

\iflong\else
\section{Full proofs}

\subsection*{Proof of Lemma~\ref{lem:linear+denial-iga-ga-lb}}
\newcommand{\tval}{\mathsf{t}}
\newcommand{\fval}{\mathsf{f}}
\newcommand{\myconstant}{\mathsf{a}}
\newcommand{\pos}{\mathit{Pos}}
\newcommand{\pol}{\mathit{Pol}}
\newcommand{\pconst}{\mathsf{p}}
\newcommand{\Ix}{I_{\vseq{x}}}
\newcommand{\Iy}{I_{\vseq{y}}}
We prove the thesis by showing a reduction from 2-QBF, inspired by the proof of Theorem~4.7 of~\cite{CM05}. We assume that the given formula is in CNF.

Let $\T=\emptyset$, and let $\P$ consist of the following EDs:
\[ 
\begin{array}{r@{\,}l}
    \forall v,t & (\K T(v,t) \ra \K\exists p\, \pos(v,p)) \\
    \forall v,p & (\K \pos(v,p) \ra \K\exists i\, C(p,i)) \\
    \forall p,t & (\K \pol(p,t) \ra \K\exists i\, C(p,i)) \\
    \forall v,i & (\K C(v,i) \ra \K\exists j\, S(i,j)) \\
    \forall i,j & (\K S(i,j) \ra \K\exists v\, C(v,j)) \\
    \forall i,j & (\K S(i,j) \ra \K\exists p,v,t\, (C(p,i) \land \pos(v,p) \land \\&\hspace{9.5em} \pol(p,t) \land T(v,t))) \\
    \forall v & (\K(T(v,\tval) \wedge T(v,\fval)) \ra \K\bot) .
\end{array}
\]
Intuitively, the first four EDs impose what follows: every variable having a truth assignment must occur in some position; every position in which a variable $v$ occurs (or having polarity $t$) must occur in some clause; every clause must have a successor.
The combination of the fourth and fifth EDs forces a loop of $S$-facts to collapse if one of them is missing.
The sixth ED imposes that, if a clause $\psi_i$ has a successor $\psi_j$, then a variable $v$ must occur in $\psi_i$ such that the truth value of $v$ matches its polarity in $\psi_i$ (i.e., $\psi_i$ is satisfied). 
Finally, the seventh ED guarantees that distinct truth values are not assigned to the same variable.

Now let $\phi=\forall \vseq{x}\,\exists \vseq{y}\,(\psi_1\wedge\ldots\wedge\psi_n)$ be a 2-QBF, where every $\psi_i$ is a clause over the propositional variables 
$\vseq{x}\cup\vseq{y}$.
Moreover, $\A=\A_1\cup\A_2\cup\{S(1,\myconstant)\}$, where:
\[
\begin{array}{r@{\;}l@{}l@{}l}
    \A_1=&\{& S(0,0)\} \,\cup \\
        &\{& C(\pconst_x^0,0), \pos(\pconst_x^0,0), \pol(\pconst_x^0,\tval), \pol(\pconst_x^0,\fval), 
        T(x,\tval), T(x,\fval) \mid x\in\vseq{x} \}\\
    \A_2=&\{& S(i,(i \bmod n)+1) \mid 1\le i\le n
    \} \,\cup \\
       &\{& C(\pconst_v^i,i), \pos(v,\pconst_v^i), \pol(\pconst_v^i,\tval), T(v,\tval) \mid 
       v \textrm{ occurs positively in } \psi_i \} \,\cup \\
       &\{& C(\pconst_v^i,i), \pos(v,\pconst_v^i), \pol(\pconst_v^i,\fval), T(v,\fval) \mid 
       v \textrm{ occurs negatively in } \psi_i \} .
\end{array}
\]
We prove that $\E \modelsga S(1,\myconstant)$ iff $\phi$ is valid, where $\E=\tup{\T,\P,\A}$.

$(\Leftarrow)$
First, for every interpretation $\Ix$ of $\vseq{x}$, let 
\[
    \A(\Ix) = \A_1 \setminus \{T(x,\fval) \mid x\in\Ix\} \setminus \{T(x,\tval) \mid x\in\vseq{x}\setminus\Ix\}.
\]

Now, let us assume that $\phi$ is not valid, and let $\Ix$ be an interpretation of $\vseq{x}$ such that there exists no interpretation $\Iy$ of $\vseq{y}$ such that $\Ix\cup\Iy$ satisfies $\psi_1\wedge\ldots\wedge\psi_n$.
We prove that $\A(\Ix)$ is an optimal GA censor of $\E$. Since $\T\cup\A(\Ix)\modelseql\P$, we only have to show that no further fact from $\clt{\A}$ (i.e.\ from $\A$) can be added to it without violating the policy $\P$.

First, notice that adding any fact from $\A_1$ to $\A(\Ix)$ would violate the last dependency of $\P$.
As for the $T$-facts contained in $\A_2$, adding one of them would either violate the last dependency as well, or cause the addition of at least one $C$-fact from $\A_2$ (because of the first dependency).
However, if any fact from $\A_2\cup\{S(1,\myconstant)\}$ with predicate $C$, $S$, $\pol$, or $\pos$ is added to $\A(\Ix)$, then the second, third, fourth and fifth dependencies would eventually require to add, for every clause $\psi_i$, the corresponding fact $S(i,(i \bmod n)+1)$ from $\A_2$.
Then, because of the sixth dependency, we would need to add, for every $\psi_i$, four facts of the form $C(p,i)$, $\pos(v,p)$, $\pol(p,t)$, $T(v,t)$ from $\A_2$ for at least one variable $v$ occurring in 
$\psi_i$.
However, due again to the last dependency of $\P$ and by construction of $\A_2$, this could be done (while preserving consistency) only if it is possible to find a truth assignment for the variables $\vseq{y}$ such that at least one variable for each clause $\psi_i$ is assigned to a value corresponding to its polarity in $\psi_i$, i.e.\ only if $\phi$ is valid, which is a contradiction.
Consequently, $\A(\Ix)\in\optcens(\E)$, which implies that $\E\not\modelsga S(1,\myconstant)$.

$(\Rightarrow)$
Conversely, assume that $\phi$ is valid. Let $\Ix$ be any interpretation of $\vseq{x}$. Then, there exists an interpretation $\Iy$ of $\vseq{y}$ such that $\Ix\cup\Iy$ satisfies $\psi_1\wedge\ldots\wedge\psi_n$.
Similarly as above, this is possible only if there exists a subset $\A'$ of $\A_2$ containing, for every $1\le i\le n$, at least five facts of the form $S(i,(i \bmod n)+1)$, $C(p,i)$, $\pos(v,p)$, $\pol(p,t)$, $T(v,t)$ 
and such that, for every variable $v$, it does not contain both $T(v,\tval)$ and $T(v,\fval)$.

It is immediate to verify that $\T\cup\A'\modelseql\P$.
Now, suppose that $S(1,\myconstant)\notin\C$ for some $\C\in\optcens(\E)$. 
Since $\C$ is optimal, this is possible only if adding $S(1,\myconstant)$ violates the fourth dependency, i.e.\ if no fact of the form $C(\_,1)$ of $\A$ occurs in $\C$.
Then, because of the fifth dependency, the fact $S(n,1)$ (i.e., the only fact of $\A$ of the form $S(\_,1)$) does not belong to $\C$.
Analogously, $\C$ cannot contain any fact of the form $C(\_,n)$, nor $S(n-1,n)$, and so on: considering the effect of the first three dependencies as well, this iterative process forces \emph{all} facts with predicate $\pos$, $\pol$, $C$ and $S$ occurring in $\A_2$ to be excluded from $\C$, other than all the facts of the form $T(y,\_)$ such that $y\in\vseq{y}$. Note then that $\C\subseteq\A_1$, and recall that $\T\cup\C\modelseql\P$ by definition of GA censor. It is now straightforward to verify that $\T\cup\C\cup\A'\modelseql\P$, thus contradicting the optimality of $\C$. Consequently, $S(1,\myconstant)$ belongs to every optimal GA of $\E$, i.e.\ $\E \modelsga S(1,\myconstant)$.

Finally, we recall that $\E\modelsga\alpha$ iff $\E\modelsiga\alpha$ when $\alpha$ is a ground atom; thus, both theses are proved.
\qed

\subsection*{Proof of Proposition~\ref{pro:iga-min-undisclosable}}
$(\Rightarrow)$
Let $\alpha\notin\censiga$, which implies that $\alpha\notin\C$ for some $\C\in\optcens(\E)$.
Since $\C$ is optimal, then $\C\cup\{\alpha\}$ is $\E$-undisclosable. Moreover, every $\subseteq$-minimal subset $S$ of $\C$ such that $S$ is $\E$-undisclosable (at least one of which does exist) contains $\alpha$.

$(\Leftarrow)$
Let $\alpha$ belong to some $\subseteq$ minimal $\E$-undisclosable subset $S$ of $\clt{\A}$.
By minimality, we have that $S\setminus\{\alpha\}$ is contained in some GA censor $\C$ of $\E$, which obviously does not contain $\alpha$. Consequently, $\alpha\notin\censiga$.
\qed

\subsection*{Proof of Proposition~\ref{pro:ga-iga-non-indistinguishable}}
Recalling Example~1 of~\cite{MRR25}, consider the CQE instance $\E=\tup{\T,\P,\A}$, where $\T=\emptyset$, $\A = \{ C(0),B(1),B(2)\,\}$, and $\P = \{ K(B(1)\land B(2)) \ra K \bot, \forall x\,(K C(x) \ra K\exists y\,B(y)) \}$.
    
Note that $\censiga$ (i.e.\ $\{C(0)\}$) is not a GA censor of $\E$ (as it does not satisfy the second ED of $\P$).
Then, consider the BCQ $q=C(0)$ (which is both GA- and IGA-entailed by $\E$), and let $\A'$ be any ABox such that $\T\cup\A'\modelseql\P$ and $\T\cup\A'\modelsga q$ (resp., $\T\cup\A'\modelsiga q$). It is immediate to see that every such $\A'$ is such that $\T\cup\A'\modelsga B(c)$ (resp., $\T\cup\A'\modelsiga B(c)$), for some constant $c$.
\qed

\subsection*{Proof of Proposition~\ref{pro:lpv-convergence}}
First, we prove property $(i)$.
Base case ($i=0$): Since $\lpv(\E,0)=\emptyset$, the thesis immediately follows.
Inductive case ($i>0$): Suppose $\lpv(\E,i)\supseteq\lpv(\E,i-1)$ and let $\alpha\in\lpv(\E,i)$. Since $\lpv(\E,i)=\lpvz(\E,\clt{\A}\setminus\lpv(\E,i-1))$, there exist $\tau\in\P$ and $\sigma\in\gs(\tau)$ such that $\T\cup\{\alpha\}\models\body(\sigma(\tau))$ and $\T\cup(\clt{\A}\setminus\lpv(\E,i-1))\not\models\head(\sigma(\tau))$. Now, since $\lpv(\E,i)\supseteq\lpv(\E,i-1)$, it follows that $\T\cup(\clt{\A}\setminus\lpv(\E,i))\not\models\head(\sigma(\tau))$, which implies that $\alpha\in\lpv(\E,i+1)$. Consequently, $\lpv(\E,i+1)\supseteq\lpv(\E,i)$.

As for property $(ii)$: From Definition~\ref{def:lpv} it immediately follows that, if $\lpv(\E,i)=\lpv(\E,i+1)$ for some $i$, then $\lpv(\E,i)=\lpv(\E,j)$ for every integer $j$ such that $j>i$.
This property and property $(i)$ imply that either $\lpv(\E,|\clt{\A}|)=\clt{\A}$ or there exists $i<|\clt{\A}|$ such that $\lpv(\E,i)=\lpv(\E,i+1)$. In both cases, it follows that, for every $n$ such that $n\geq|\clt{\A}|$, $\lpv(\E,n)=\lpv(\E,n+1)$.
\qed

\subsection*{Proof of Proposition~\ref{pro:lpv-iga-censor}}
It is easy to see that every $\alpha$ of Definition~\ref{def:lpv} is such that $\{\alpha\}$ is a $\subseteq$-minimal $\E$-undisclosable subset of $\clt{\A}$. Then, the thesis follows by Proposition~\ref{pro:iga-min-undisclosable}.
\qed

\subsection*{Proof of Theorem~\ref{thm:iga-linear-ptime-complete}}
\newcommand{\pv}{Vars}
\newcommand{\bv}{BVars}
We prove the lower bound by showing a reduction from the \textsc{Horn-Sat} problem.

Let $\phi$ be a set of ground Horn rules, and let us assume w.l.o.g.\ that $\phi$ contains at least one headless rule (i.e.\ a clause with only negated variables). Let $\phi'$ be another set obtained from $\phi$ by slightly modifying rules without heads as follows: every headless rule $r$ is replaced by the rule with the same body as $r$ and with the new variable $u$ in the head; moreover, for every propositional variable $p$ occurring in $\phi'$ (including $u$), we add a rule $p \ra p$ to $\phi'$. Notice that, by construction, $\phi'$ always contains at least two rules with $u$ in the head. It is immediate to verify that $\phi$ is unsatisfiable iff $u$ belongs to all the models (and hence to the minimal model) of $\phi'$.
 
Now, we associate an identifier $r_i^x$ to every Horn rule in $\phi'$ having the variable $x$ in its head.
Then, let $h[x]$ be the number of rules of $\phi'$ having the variable $x$ in their head, let $\pv(\phi')$ be the set of propositional variables occurring in $\phi'$ and, for every rule $r_i$, let $\bv(r_i)$ be the set of variables in the body of $r_i$.
We define the ABox $\A$ as the set of facts:
\[
\begin{array}{r@{\,}l}
    \displaystyle
    \bigcup_{r_i\in\phi'} & \{ B(r_{i},x) \mid x \in\bv(r_i) \} \:\cup \\
    \displaystyle
    \bigcup_{x\in\pv(\phi')}\bigcup_{1\le i\le h[x]} & \{  H(r_i^x,x), S(r_i^x,r_j^x) \mid  j=(i \bmod h[x])+1 \}
\end{array}
\]
Moreover, we set $\T=\emptyset$, $q=\exists r\,H(r,u)$ and $\P$ to the following set of linear EDs:
\[
\begin{array}{r@{\,}l}
    \forall r,r'\! & (\K S(r,r')\rightarrow \K\exists r'',v \, S(r',r'') \land H(r,v)) \\
    \forall r,v & (\K H(r,v)\rightarrow \K\exists r'\,S(r,r')) \\
    \forall r,v & (\K H(r,v)\rightarrow \K\exists v'\,B(r,v')) \\
    \forall r,v & (\K B(r,v)\rightarrow \K\exists r' \, H(r',v)) \\
\end{array}
\]
Observe that:
\begin{itemize}
    \item The first two EDs imply that the deletion of any fact of the form $H(r,v)$ causes the deletion of all the facts of the form $H(\_,v)$ and all the $S$-facts related to the EDs having $v$ in their head. 
    \item The third ED implies that, if all the $B$-facts for the body variables of a rule are deleted, then also the $H$-fact for its head variable must be deleted.
    \item The fourth ED implies that, if all the $H$-facts for a variable are deleted, then also all its $B$-fact must be deleted.
\end{itemize}

The set $\censiga$ of $\E=\tup{\T,\P,\A}$ can be obtained starting from $\clt{\A}$ (i.e.\ from $\A$, because $\T=\emptyset$) and deterministically removing some facts according to the EDs of $\P$.
We now prove that, for every variable $x\in\pv(\phi')$, no fact of the form $H(\_,x)$ belongs to $\censiga$ (i.e.\ $\E\not\modelsiga q$, which for linear EDs holds iff $\E\not\modelsga q$) iff $x\in I$, where $I$ is the minimal model of $\phi'$.

\smallskip
$(\Leftarrow)$ 
Let $x\in I$. Note that this is possible only if $x$ occurs in the head of a rule $r\in\phi'$ that either does not have a body (\emph{unit clause}) or whose body variables all belong to $I$.
\begin{itemize}
    \item In the first case, due to the first three EDs, $\censiga$ does not contain any atom of the form $H(\_,x)$.
    \item The second case occurs instead when all variables $b_1,\ldots,b_m\in\bv(r)$ either occur in unit clauses themselves or, in turn, occur as head variables in rules whose body variables belong to $I$. By recursively applying the previous and the current point, respectively, one can see that all facts \(
        B(r,b_1),\ldots,B(r,b_m)
    \) are removed from $\censiga$ because of the fourth ED (possibly combined with the first two). 
    As a result, by the first three EDs and since $x$ occurs in the head of $r$, it follows that no fact of the form $H(\_,x)$ belongs to $\censiga$ either.
\end{itemize}

$(\Rightarrow)$ 
Let $x$ be such that no fact of the form $H(\_,x)$ belongs to $\censiga$. Note that, since every variable occurs in the head of some rule (by the assumption that a rule $p\ra p$ occurs in $\phi'$ for every $p\in\pv(\phi')$), then there exists a fact $H(\_,x)\in\A$. The absence of such $H$-facts from $\censiga$ can only be due to a violation of the third ED (possibly combined with the first two), i.e.\ all the $B(r,\_)$ facts (for some rule $r=b_1,\ldots,b_m \ra x\in\phi'$) are missing from $\censiga$.

Then, we either have that $r$ is a unit clause of $\phi'$ (which would imply that $x\in I$) or that every fact $B(r,b_i)$ (for $1\le i\le m$) occurring in $\A$ has been removed for building $\censiga$. These last removals can only be due to a violation of the fourth ED (possibly combined with the first two), i.e.\ all facts $H(\_,b_i)$ (for every $1\le i\le m$) are missing from $\censiga$. 
By recursively applying this point, we conclude that all $b_i$ variables belong to $I$ (implying that also $x\in I$).

By instantiating the above property with $x=u$, we have that $u\in I$ (i.e.\ $\phi$ is unsatisfiable) iff $\E\modelsiga q$.
\qed


\subsection*{Proof of Proposition~\ref{pro:mpv-convergence}}
From Definition~\ref{def:mpv}, and in a way analogous to the proof of Proposition~\ref{pro:lpv-convergence}, we get the following properties:
\begin{enumerate}[label=$(\roman*)$]
    \item for every $i$ such that $1\leq i\leq n$, $\mpv(\E,i+1)\supseteq\mpv(\E,i)$;
    \item if $\mpv(\E,i+1)=\mpv(\E,i)$, then $\mpv(\E,j)=\mpv(\E,i)$ for every $j>i$. 
\end{enumerate}

Now, two cases are possible:
\begin{enumerate}
    \item there exists an integer $i$ such that $0\leq i\leq |\clt{\A}-1|$ and $\mpv(\E,i+1)=\mpv(\E,i)$. In this case, from the above property $(ii)$ it follows that $\mpv(\E,n)=\mpv(\E,n+1)$ for every integer $n$ such that $n\geq|\clt{\A}|$; 
    \item if $\mpv(\E,i+1)\supset\mpv(\E,i)$ for every integer $i$ such that $0\leq i\leq |\clt{\A}-1|$ (i.e.\ every $\mpv(\E,i+1)$ has at least one more atom than $\mpv(\E,i)$), then $|\mpv(\E,|\clt{\A}|)|\geq|\clt{\A}|$, and since $\mpv(\E,i)$ is by definition a subset of $\clt{\A}$, it follows that 
    $\mpv(\E,|\clt{\A}|)=\clt{\A}$.
    For the same reason, it follows that $$\mpv(\E,|\clt{\A}|+1) \subseteq \clt{\A} \subseteq \mpv(\E,|\clt{\A}|).$$ On the other hand, by the above property $(i)$ it follows that $$\mpv(\E,|\clt{\A}|+1)\supseteq\mpv(\E,|\clt{\A})|.$$ Consequently, $\mpv(\E,|\clt{\A}|+1)=\mpv(\E,|\clt{\A}|)$. Finally, from the above property $(ii)$ it follows that $\mpv(\E,j)=\mpv(\E,|\clt{\A}|)$ for every integer $j$ such that $j\geq |\clt{\A}|$.
\end{enumerate}
\qed

\subsection*{Proof of Proposition~\ref{pro:mpv-censor}}

Let $\C$ denote the set $\clt{\A}\setminus\A'$. Obviously, $\C\cap\A'=\emptyset$.
By contradiction, let $\T\cup\C\not\modelseql\P$. 
Now, let $\A''$ be the $\subseteq$-minimal subset of $\C$ such that, for some $\tau\in\P$ and $\sigma\in\gs(\tau)$, $\T\cup\A''\models\body(\sigma(\tau))$ and $\T\cup\C\not\models\head(\sigma(\tau))$ (notice that such $\A''$ exists because $\T\cup\C\not\modelseql\P$).
Since $\C\subseteq\clt{\A}$ (by construction), then $\A''$ is also a $\subseteq$-minimal subset of $\clt{\A}$ such that, for some $\tau\in\P$ and $\sigma\in\gs(\tau)$, $\T\cup\A''\models\body(\sigma(\tau))$ and $\T\cup\C\not\models\head(\sigma(\tau))$, i.e.\ $\A''\subseteq\mpvz(\E,\C)$. However, since $\mpvz(\E,\C)\subseteq\A'$ (by hypothesis), this contradicts the fact that $\C\cap\A'=\emptyset$, thus proving property $(i)$.

\smallskip
Then, by definition of $\mpv(\E)$, we have $\mpv(\E)=\mpvz(\E,\clt{\A}\setminus\mpv(\E))$, hence by property $(i)$ we get property $(ii)$.

\smallskip
Finally, let $\A'$ be any set of facts such that $\A'\supseteq\mpvz(\E,\clt{\A}\setminus\A')$. We prove by induction that, for every $i$ such that $1\leq i\leq |\clt{\A}|$, $\A'\supset\mpv(\E,i)$.
\begin{itemize}
    \item Base case ($i=0$): trivially, $\A'\supseteq\emptyset=\mpv(\E,0)$.
    \item Inductive case: suppose $\A'\supseteq\mpv(\E,i)$. Notice that, for every pair of sets $S,S'$, the set $\mpvz(\E,S)$ monotonically decreases when $S$ increases, meaning that $\mpvz(\E,S)\supseteq\mpvz(\E,S')$ whenever $S\subseteq S'$. This holds because, by Definition~\ref{def:mpv}, the MPVs are searched in both cases among all possible subsets of $\clt{\A}$, but the non-entailment check is done on $S$ and $S'$, respectively. Therefore, $\mpvz(\E,\clt{\A}\setminus\A')\supseteq\mpvz(\E,\clt{\A}\setminus\mpv(\E,i))$, and since $\mpv(\E,i+1)=\mpvz(\E,\clt{\A}\setminus\mpv(\E,i))$, it follows that $\A'\supseteq\mpv(\E,i+1)$.
\end{itemize}
Since $\mpv(\E)=\mpv(\E,|\clt{\A}|)$, it follows that $\A'\supseteq\mpv(\E)$, thus proving property $(iii)$.
\qed

\subsection*{Proof of Theorem~\ref{thm:mpv-iga-denials}}

Note that, in the hypothesis $(i)$, the fixpoint for $\mpv(\E,i)$ is already reached for $i=1$ (i.e.\ $\mpv(\E,1)=\mpv(\E)$).
Then, $\mpv(\E)$ coincides with the sets of $\subseteq$-minimal $\E$-undisclosable subsets of $\clt{\A}$. Hence, the thesis follows by Definition~\ref{def:mpv-censor} and Proposition~\ref{pro:iga-min-undisclosable}.

Under the hypothesis $(ii)$, Definition~\ref{def:lpv} and Definition~\ref{def:mpv} coincide; thus, the thesis follows immediately.
\qed

\subsection*{Proof of Proposition~\ref{pro:alg-mpv-ent-dlliter}}
First, recall that, for every $\tau\in\P$, $\ucqrewr(\body(\tau),\T)$ returns a UCQ $q'$ such that for every ABox $\A$ and for every $\sigma\in\gs(\tau,\Const(\A))$, $\T\cup\A\models\body(\sigma(\tau))$ iff $\A\models\sigma(q')$.

Then, we prove the following property (*): at every iteration $i$ of the repeat--until loop of the algorithm, the set of sets $\I$ computed after the minimization step of line~\ref{step:minimality-check} is the set of all the minimal policy violations of $\E$ with respect to $\A_p'$.
%
Suppose $\A''$ is a minimal policy violation of $\E$ with respect to $\A_p'$. Then, there exist $\tau\in\P$ and $\sigma\in\gs(\tau,\Const(\A))$ such that $\T\cup\A''\models\body(\sigma(\tau))$ and $\T\cup(\C\setminus\A_p')\not\models\head(\sigma(\tau))$. Consequently, $\A''\models\sigma(q')$, where $q'$ is the UCQ returned by $\ucqrewr(\body(\tau),\T)$. Consequently, $\A''$ is a homomorphic image of $\ucqrewr(\body(\tau),\T)$ in $\clt{\A}$, hence $\A''\in\I$ before the minimization of $\I$.
Furthermore, it is immediate to verify that every set $\A''$ that belongs to the set of sets $\I$ before its minimization contains a minimal policy violation of $\E$ with respect to $\A_p'$. These two properties imply that, before its minimization, the set of sets $\I$ computed by the algorithm contains all the minimal policy violations of $\E$ w.r.t.\ $\A_p'$ and supersets of them. Consequently, property (*) follows.

Finally, property (*) immediately implies that the set $\A'$ computed by the algorithm at the $i$-th iteration corresponds to $\mpv(\E,i)$. Consequently, the set $\C$ computed by the algorithm corresponds to $\corecens$, which implies the thesis.
\qed

\clearpage
\fi
\renewcommand{\subsubsection}[1]{\medskip\noindent\textbf{#1}\\[1mm]\noindent}
\renewcommand{\paragraph}[1]{\medskip\noindent\textbf{#1}\;}
\renewcommand{\subparagraph}[1]{\medskip\noindent\textit{#1}\;}

\section{Practical implementation and experiments}
\noindent
This section provides further details on our implementation and experimental evaluation.



\subsection{System architecture and implementation setting}


The ABox is stored in a MySQL database, where each predicate is represented by a table whose schema reflects the arity of the predicate. Unary and binary predicates are mapped to single-column (\texttt{attr1}) and two-column (\texttt{attr1}, \texttt{attr2}) tables, respectively. 



The interaction between the application layer and the database is mediated by a SQL generation component, which translates logical formulas into executable SQL queries, which are used both for manipulating the data (i.e.\ for computing the (A)MPV censor) and for query answering.




\paragraph{Input, output, and execution setup.}
The implemented software takes as input a TBox expressed as an OWL~2~QL ontology,
a confidentiality policy specified through a JSON file, and an ABox stored in a relational database.
In the implementation, the ABox is accessed through a database connection.
The (A)MPV censor is stored in a new database instance. 

Before executing the algorithm for computing the (A)MPV censor, the software loads and preprocesses both the ontology and the policy.
The TBox is loaded and prepared in order to support ABox and policy expansion.
In particular, during the TBox preprocessing phase, data properties are removed from the input ontology, as they are not part of $\dlliter$, and OWL axioms that are not supported by the tree-witness library are rewritten as follows:
\begin{itemize}
    \item OWLEquivalentClassesAxiom axioms are expressed using two instances of OWLSubClassOfAxiom (i.e., the equivalence between concepts $A$ and $B$ is expressed as $A \sqsubseteq B$ and $B \sqsubseteq A$);
    \item OWLEquivalentObjectPropertiesAxiom axioms are expressed using multiple instances of OWLSubObjectPropertyOfAxiom (i.e., the equivalence between roles $R$ and $S$ is expressed as $R \sqsubseteq S$ and $S \sqsubseteq R$);
    \item OWLSymmetricObjectPropertyAxiom axioms are expressed using an OWLInverseObjectPropertiesAxiom and an OWLSubObjectPropertyOfAxiom (i.e., role symmetry is expressed as $R^- \sqsubseteq R$).
\end{itemize}

The policy is parsed from the JSON file and translated into a set of \texttt{EpistemicDependency} Java objects.

\begin{example}
The linear ED
\[
\forall x\,(\K\mathsf{Employee}(x) \rightarrow \K \exists y\,\mathsf{worksFor}(x,y))
\]
can be specified in the JSON policy file as follows:
\begin{verbatim}
{
  "body": "Q(x) :- :Employee(x).",
  "head": "Q(x) :- :worksFor(x,y)."
}
\end{verbatim}

\qedex
\end{example}



\paragraph{Structure of the algorithm.}
The (A)MPV censor is constructed by iteratively modifying a working copy of the database, removing ABox assertions that violate the confidentiality policy.
In particular, the algorithm is organized into three main phases:
\begin{enumerate}[label=$(\roman*)$]
    \item initialization phase, which prepares the input data and policy;
    \item iterative phase, in which policy violations are detected and the corresponding tuples are removed;
    \item final phase, in which a fixpoint is reached and no further deletions are performed.
\end{enumerate}

\paragraph{Initialization phase.}
The initialization phase consists of a sequence of steps, that prepare the data, the ontology, and the policy before the execution of the iterative censorship phase. In the following, we analyze each of these steps.

\subparagraph{ABox cloning.}
As a first step of the initialization phase, a working copy of the database storing the ABox is created.
For each table of the original database, the corresponding table structure is recreated in the cloned database, and all tuples are copied.
This approach ensures that the censorship process operates on an isolated copy of the data.
As a result, the original ABox remains unchanged, while the algorithm can perform arbitrary data manipulation.


\subparagraph{ABox expansion.}
The next step includes the expansion of the ABox with respect to the TBox, i.e.\ the computation of $\clt{\A}$.
The expansion is performed predicate-wise.
For each concept and object property occurring in the TBox, an atomic rewriting with respect to the ontology is computed by applying $\atomrewr$~\cite{CLRS24} (a simplified implementation of $\ucqrewr$ for atomic queries) to the corresponding predicate atom.
Such a rewriting captures all the possible ways in which instances of the corresponding predicate can be inferred from existing ABox assertions and TBox axioms.

\begin{example}\label{example: expansion-abox}
Consider the following TBox:
\[
\begin{aligned}
\T =
\{\, 
& \exists\,\mathsf{attends} \sqsubseteq \mathsf{Student},\\
& \exists\,\mathsf{attends}^{-} \sqsubseteq \mathsf{Course}
\,\}
\end{aligned}
\]

By applying the atomic rewriting procedure $\atomrewr$ to the predicate
$\mathsf{Student}(x)$ with respect to the TBox, we obtain the following rewriting:
\[
\mathsf{Student}(x) \;\lor\; \exists y\,\mathsf{attends}(x,y).
\]
Intuitively, this rewriting captures the fact that any individual occurring as the
first argument of the role $\mathsf{attends}$ must be an instance of
$\mathsf{Student}$.

Similarly, applying $\atomrewr$ to the predicate $\mathsf{Course}(x)$
yields the rewriting:
\[
\mathsf{Course}(x) \;\lor\; \exists y\,\mathsf{attends}(y,x),
\]
reflecting the fact that individuals occurring as the second argument of
$\mathsf{attends}$ must be instances of $\mathsf{Course}$.
\qedex
\end{example}

Each rewritten atom is then translated into an SQL query and executed on the working database.
The resulting tuples are inserted into the corresponding table, provided that they are not already present.
In this way, the ABox is incrementally saturated with respect to the TBox, while avoiding duplicate facts.
After this step, the database contains a materialized representation of the ABox closure under the TBox.

\begin{example}
We refer to Example~\ref{example: expansion-abox} and we consider the following ABox:
\[
\begin{array}{r@{}l}
\A = \{&\mathsf{Course}(\mathsf{courseA}),\;
\mathsf{attends}(\mathsf{Bob}, \mathsf{courseA}), \\&\mathsf{Student}(\mathsf{Alice})\}.
\end{array}
\]

The rewriting of $\mathsf{Student}(x)$ is translated into the following SQL query, which retrieves all individuals that explicitly occur as instances of \texttt{Student} or that can be inferred via the role \texttt{attends}:
\begin{verbatim}
SELECT DISTINCT x FROM (
  SELECT attr1 AS x FROM attends
  UNION ALL
  SELECT attr1 AS x FROM student
) AS q0
\end{verbatim}

When executed on the working database, this query returns the values \texttt{Bob} and \texttt{Alice}, identifying the individuals that must be considered as instances of the predicate \texttt{Student} according to the rewriting.

To materialize these inferred assertions in the ABox, the result of the query is used to construct an \texttt{INSERT} statement that adds the corresponding tuples to the table associated with \texttt{Student}, provided that they are not already
present. This is achieved by adding to the query a \texttt{NOT EXISTS} condition, as shown below:
\begin{verbatim}
INSERT INTO student(attr1)
SELECT DISTINCT x FROM (
  SELECT attr1 AS x FROM attends
  UNION ALL
  SELECT attr1 AS x FROM student
) AS q0
WHERE NOT EXISTS (
  SELECT 1 FROM student t
  WHERE t.attr1 = q0.x
)
\end{verbatim}
As a result, only the fact $\mathsf{Student}(\mathsf{Bob})$ is inserted into the ABox during the expansion phase.
\qedex
\end{example}

\subparagraph{Policy expansion.}
The next step of the initialization phase consists in expanding the confidentiality policy with respect to the TBox.
In the implementation, the body of each ED is rewritten with respect to the TBox by exploiting the tree-witness rewriting algorithm. 
For every rewritten version of the body, a new ED is generated by combining the rewritten body with the original head of the dependency. During this process, the universally quantified variables are aligned with those of the original dependency. The result of this step is a new expanded policy $\P'$ that captures all potential policy violations induced by the TBox.
\begin{example}
    Consider the following TBox $\T$ and policy $\P$:
    \[
    \begin{aligned}
    \T =\ 
    & \{\, 
        B \sqsubseteq C
    \,\},
    \\
    \P =\ 
    & \{\, 
        \forall x\,(
            \K (C(x) \wedge A(x)) \rightarrow \K D(x)
        )
    \,\}.
    \end{aligned}
    \]
    The policy $\P'$ obtained after the policy expansion step is the following:
    \[
    \begin{aligned}
    \P' =\ 
    & \{\, 
        \forall x\,(
            \K (C(x) \wedge A(x)) \rightarrow \K D(x)
        ) \\
    &\quad     
        \forall x\,(
            \K (B(x) \wedge A(x)) \rightarrow \K D(x)
        ) 
    \,\}.
    \end{aligned}
    \]
    \qedex
\end{example}

We remark that such a policy expansion step is used to implement the check $\T\cup\A\models\body(\sigma(\tau))$ done at line~\ref{step:ed-violation} of Algorithm~\ref{alg:mpv-ent-dlliter} (see the next paragraph ``Further preprocessing steps").

\subparagraph{Head expansion.}
As a final step of the initialization phase, also the rewriting of the head of every ED is precomputed, for optimization purposes.

\subparagraph{Further preprocessing steps.}
For each ED in the expanded policy, the algorithm constructs a temporary table representing the closure of its body with respect to the initial ABox and the TBox.
This table, called \texttt{bodyBase}, contains all instantiations of the body that are implied by the ontology (i.e.\ the homomorphic images described in Section~\ref{sec:experiments}).
To this end, the body of each ED $\tau$ is 
compiled into an SQL query, which is executed on the working database, and its result is materialized into the corresponding \texttt{bodyBase} table.

\begin{example}
\label{example-bodybase}
    Consider the following ED:
    \[
    \tau:\quad
    \forall x,y\,
    \bigl(
    \K\,(\mathsf{attends}(x,y) \wedge \mathsf{Course}(y))
    \rightarrow
    \K\,\mathsf{Student}(x)
    \bigr).
    \]
    Assume the following TBox:
    \[
    \T =
    \{\, 
    \exists\,\mathsf{attends} \sqsubseteq \mathsf{Student}
    \,\}.
    \]
    Let the initial ABox be:
    \[
    \A =
    \{\, 
    \mathsf{attends}(\mathsf{Bob}, \mathsf{courseA}),
    \mathsf{Course}(\mathsf{courseA})
    \,\}.
    \]
    
    During the preprocessing step, the body of $\tau$ is first translated into the formula $\mathsf{attends(x,y)} \wedge \mathsf{Course(y)}$ that is then compiled into the following SQL query:
    \begin{verbatim}
    SELECT x, y 
    FROM 
      ( SELECT attr1 AS y
        FROM `course` Course1 ) q0
      NATURAL JOIN
      ( SELECT attr1 AS x, attr2 AS y
        FROM `attends` attends1 ) q1
    \end{verbatim}
    The evaluation of this query over the working database retrieves all instantiations of the body that are implied by $\T \cup \A$. 

    The result of this query is materialized into the temporary table
    \texttt{bodyBase}, which in this case contains the tuple
    $(x=\mathsf{Bob},\, y=\mathsf{courseA})$.
    \qedex
\end{example}

The \texttt{bodyBase} table associated with each ED $\tau$ is computed once during preprocessing, and it is never modified during the iterative censorship process.

Observe that, for every $\sigma\in\gs(\tau)$ and for every subset $\A''\in\hi(\body(\sigma(\tau)))$, there exists a corresponding tuple in the \texttt{bodyBase} table associated with $\body(\tau)$ (and vice versa).
Now, the previous policy expansion step guarantees that, for every $\tau\in\P$ and for every CQ $q$ such that $q\in q'$ where $q'$ is the UCQ returned by $\ucqrewr(\body(\tau),\T)$, there exists an ED $\tau'$ in the policy expansion of $\P$ with respect to $\T$ such that $\tau'$ is obtained from $\tau$ by replacing its body with $q$. Therefore, there exists a \texttt{bodyBase} table associated with each such $q$.
Consequently, for every $\tau\in\P$, for every $q\in q'$, where $q'$ is the UCQ returned by $\ucqrewr(\body(\tau),\T)$, for every $\sigma\in\gs(\tau)$ and for every $\A''\in\hi(\sigma(q))$, there exists a tuple representing $\A''$ in the \texttt{bodyBase} table associated with $q$. This guarantees that all the potential MPVs needed to construct the MPV censor are correctly represented in the \texttt{bodyBase} tables.

\paragraph{Iterative phase.}
At a high level, the iterative phase of the (A)MPV censor is built as follows:
$(i)$ for each ED, a working temporary table is constructed containing the current instantiations of its body;
$(ii)$ tuples that satisfy the head are removed from this table;
$(iii)$ the remaining tuples correspond to violations and are deleted from the ABox;\footnote{For the MPV censor, this step is preceded by the minimality check, which here is not described, and it consists in collecting in Java all the results of the temporary table containing the candidate violations, mutually compare (w.r.t.\ set containment) their representation as sets of facts, and use the minimal sets for the deletion process (back with SQL).}
$(iv)$ the process is repeated until a fixpoint is reached.

We now analyze each algorithm step, highlighting the used data structures and their role in enforcing the policy.

\subparagraph{Construction of working body tables.}
At the beginning of each iteration, for every ED $\tau \in \P'$, the algorithm constructs a working temporary table, denoted as \texttt{bodyWork}, starting from the corresponding \texttt{bodyBase} table (which remains fixed throughout the whole execution).
The \texttt{bodyWork} table is created as a fresh copy of \texttt{bodyBase} and represents the current set of candidate body instantiations that may still lead to a policy violation in the current iteration.

\subparagraph{Identification of violating tuples.}
For each ED $\tau \in \P'$, the algorithm retrieves the corresponding \texttt{bodyWork} table and considers the set of CQs obtained from the rewriting $\ucqrewr(\head(\tau),\T)$ of its head with respect to the TBox.
Each of such CQs represents a sufficient condition under which the head of $\tau$ is satisfied.
Operationally, for each $q\in\ucqrewr(\head(\tau),\T)$, the corresponding SQL query is created and evaluated against the current working ABox. Specifically, tuples in \texttt{bodyWork} whose variable assignments are such that one of such CQs holds in the current ABox are removed from the table, as they do not constitute a policy violation.

\begin{example}
\label{example-tupleviolationmpv}
Consider the TBox $\T=\{\exists\mathsf{hasAlumnus}^- \ISA \mathsf{Student}\}$ and the ED
\[
    \tau_0 : \forall x,y\,\bigl(
    \K( \mathsf{Course}(y) \wedge \mathsf{attends}(x,y) )
    \rightarrow
    \K\,\mathsf{Student}(x)
    \bigr).
\]

In this case, the head of $\tau_0$ is rewritten into the disjunctive formula
$ \mathsf{Student}(x)
\lor
\exists b\,\mathsf{hasAlumnus}(b,x)$,
which is actually represented as a set:
\[
headRew =
\bigl\{
\mathsf{Student}(x),\;
\exists b\,\mathsf{hasAlumnus}(b,x)
\bigr\}.
\]
%
Now consider the following ABox $\A$:
\[
\begin{array}{r@{}l}
    \A =
    \{&
    \mathsf{attends}(\mathsf{s1},\mathsf{c1}),
    \mathsf{Course}(\mathsf{c1}),
    \mathsf{Student}(\mathsf{s1}),
    \\&
    \mathsf{attends}(\mathsf{s2},\mathsf{c1})
    \}.
\end{array}
\]

The $\texttt{bodyWork\_0}$ table associated with $\tau_0$ has two columns, corresponding to the body variables $x$ and $y$, and contains the tuples $(x=\mathsf{s1},\, y=\mathsf{c1})$ and $(x=\mathsf{s2},\, y=\mathsf{c1})$, representing the two instantiations of the body of $\tau_0$ in the current ABox.

Since the assertion $\mathsf{Student}(\mathsf{s1})$ is entailed by $\T \cup \A$, the first head rewriting $\mathsf{Student}(x)$ is satisfied for the substitution $x=\mathsf{s1}$.
As a consequence, the body instantiation $(x=\mathsf{s1},\, y=\mathsf{c1})$ does not constitute a policy violation and must be removed from \texttt{bodyWork\_0}.

Operationally, this filtering step is implemented by joining
\texttt{bodyWork\_0} with the SQL query corresponding to the head rewritings on the shared variable $x$, and deleting all matching
tuples.
For instance, the following deletion query removes the satisfied body instantiation:
\begin{verbatim}
DELETE bodyWork_0 
FROM bodyWork_0
JOIN (
  SELECT attr1 AS x FROM student
) AS myAlias
ON bodyWork_0.x = myAlias.x
\end{verbatim}
Analogously, the same procedure is applied to the second head rewriting:
\begin{verbatim}
DELETE bodyWork_0 
FROM bodyWork_0
JOIN (
  SELECT attr2 AS x FROM hasalumnus
) AS myAlias 
ON bodyWork_0.x = myAlias.x
\end{verbatim}
After executing the above queries, the tuple $(x=\mathsf{s2},\, y=\mathsf{c1})$ remains in
\texttt{bodyWork\_0}, as neither $\mathsf{Student}(\mathsf{s2})$ nor $\mathsf{hasAlumnus}(b,\mathsf{s2})$ can be inferred from the current ABox.
\qedex
\end{example}


\subparagraph{Violating tuples deletion.}
At this point, for each ED, we have an associated \texttt{bodyWork} table, whose tuples correspond to body instantiations that violate the ED.

In this step, we aim to delete the corresponding ground assertions of such violations from our ``temporary'' ABox. To do that, for each ED $\tau \in \P'$, the algorithm retrieves the associated \texttt{bodyWork} table and, for each atom occurring in its body, it constructs a join-based \texttt{DELETE} statement that removes from the corresponding ABox table all tuples whose attribute values match the violating instantiations stored in \texttt{bodyWork}.
The join condition is built by aligning the variables of the body atom with the columns of the ABox table, while constants are handled through equality conditions.

\begin{example}\label{example-corega-deletion}
We continue Example~\ref{example-tupleviolationmpv}. 
After the identification of violating tuple step, the table \texttt{bodyWork\_0} contains exactly the body instantiations that violate the ED $\tau_0$.

In the violating tuples deletion step, the algorithm removes from the ABox all ground assertions that participate in this violation.
Since the body of $\tau_0$ consists of the atoms
$\mathsf{Course}(y)$ and $\mathsf{attends}(x,y)$, the algorithm processes each body atom separately.

For the atom $\mathsf{attends}(x,y)$, the following \texttt{DELETE} statement is constructed and executed:
\begin{verbatim}
DELETE attends 
FROM attends JOIN bodyWork_0 tmp 
   ON tmp.x = attends.attr1 
  AND tmp.y = attends.attr2
\end{verbatim}
This query removes the violating tuple $\mathsf{attends} (\mathsf{s2},\mathsf{c1})$ from the ABox.

Analogously, for the atom $\mathsf{Course}(y)$, the following \texttt{DELETE} statement is constructed and executed:
\begin{verbatim}
DELETE course 
FROM course JOIN bodyTemp_1 tmp 
ON tmp.y = course.attr1
\end{verbatim}
This query removes the violating tuple $\mathsf{Course} (\mathsf{c1})$ from the ABox.

After executing these deletions, all ground assertions participating in the policy violation have been removed, and the censored ABox is $\A = \{\mathsf{attends(s1,c1), Student(s1)}\}$.

\qedex
\end{example}

\paragraph{Final phase.}
The final phase corresponds to the termination of the iterative censorship process.
The algorithm stops when an iteration completes without performing any further deletions on the working database,\footnote{Java APIs return the number of deleted tuples after a \texttt{DELETE} query.} meaning that a fixpoint has been reached. 
At this point, the remaining ABox assertions no longer violate the confidentiality policy under the considered semantics, and the working database corresponds to the (A)MPV censor.

\subsection{Optimizations}
The algorithm for computing the AMPV censor requires several iterations over the ABox and repeated evaluations of EDs.
In order to make the approach scalable to large datasets, we implemented a number of optimizations aimed at reducing redundant computations, simplifying SQL queries, and improving database performance.
One of the main optimizations is constituted by the technique for tuple identification and deletion described above. 
Further optimization techniques adopted in our implementation are listed below.

\paragraph{Evaluation of active EDs.}
A first optimization concerns the identification of the EDs that need to be evaluated in the current iteration of the censoring algorithm.
A naive implementation would re-evaluate all EDs at every iteration, even when no relevant changes have occurred in the underlying ABox, leading to a significant amount of redundant computation.
To avoid this, we introduce the notion of \emph{active EDs}.
At the beginning of the algorithm, all dependencies $\tau \in \P'$ are considered active.
During each iteration, in the violating tuples deletion step, we track the body predicates whose corresponding ABox tables are affected by tuple deletions. 
Intuitively, only those EDs whose head mentions affected predicates may become newly violated in the next iteration.
At the end of each iteration, the set of affected predicates is used to select only those EDs whose (expanded) heads contain at least one affected predicate.
All other dependencies are safely ignored in the next iteration.

\paragraph{Database indexes.}
A further important optimization concerns the interaction with the underlying relational database.
The censoring algorithm repeatedly executes join-based $\texttt{DELETE}$ statements and selection queries over ABox tables. Without appropriate indexing, these operations require full table scans, which become too expensive as the size of the data grows.
To mitigate this issue, we automatically create indexes on all columns of the cloned database tables before executing the censoring procedure. Indexes allow the database engine to quickly locate tuples matching join conditions or selection predicates, thereby reducing the cost of query execution.
Although index creation introduces a small overhead in the initialization phase, this cost is amortized over the many SQL operations performed. 

\subsection{Experiments}
In this section, we present more details about the experimental evaluation of our implementation of the proposed framework.
In particular, we analyze the scalability of the AMPV censor with respect to the size of the ABox, analyze its iterative behavior, assess the impact of the design choices and optimizations discussed in the previous sections, and evaluate the cost of query answering over the censored ABox.

All experiments are conducted by varying the size of the ABox while keeping the TBox and the policy fixed. 

\subsubsection{Experimental setup}
All experiments were conducted on a laptop with an Intel Core i7-7700HQ processor running at 2.80~GHz and 8~GB of RAM. The implementation is written in Java 11 and relies on MySQL 8.0 for storing and manipulating the ABox.

\paragraph{Knowledge base.}
We rely on the OWL2Bench benchmark for OWL ontologies~\cite{owl2bench-queries} ($\sim$140 classes and $\sim$890 object properties, respectively concept and roles), which models a university domain and provides a data generator for producing ABoxes of customizable size.
In particular, we consider ABoxes generated with $N = 1, 5, 10, 20, 25,$ and $50$ universities.
As $N$ increases, the size of the ABox grows accordingly, allowing us to evaluate the scalability of the censoring algorithm with respect to the number of assertions.
Table~\ref{tab:core-ga-censor-summary} reports the size of the ABoxes used in the experiments, both before and after expansion with respect to the TBox.

\paragraph{Policy.}
The policy used in the experimental evaluation consists of a fixed set of 11 EDs. The complete list of EDs used in the experiments is reported in Appendix~\ref{app:policy}.
The dependencies are arbitrary and capture heterogeneous structural patterns (e.g., linear and non-linear EDs), interdependent violations requiring multiple iterations of the censoring algorithm, and non-trivial interactions with the data.
They are semantically meaningful in the OWL2Bench university domain and are not trivially entailed by the TBox.
After expansion with respect to the TBox, the policy results in a total of 153 EDs.

\paragraph{Queries.}
For the query answering evaluation, we rely on the OWL2Bench QL benchmark query set, consisting of ten SPARQL queries (reported in Appendix~\ref{app:queries}), fixed across all experimental configurations. 
For each dataset size $N \in \{1,5,10,20,25,50\}$, queries are evaluated over the censored ABox produced by the (A)MPV censor.
Each query is first rewritten with respect to the TBox using the tree-witness rewriting technique, and the resulting SQL query is then executed over the relational database storing the censored ABox.

\begin{table*}[t]
\centering
\setlength{\tabcolsep}{2pt}

\begin{tabular}{l|ccc|cccccc|}
\cline{2-10}
\multicolumn{1}{c|}{\multirow{3}{*}{}}
& \multicolumn{9}{c|}{\textbf{Dataset size (N)}} \\
\cline{2-10}
& \multicolumn{3}{c|}{\textbf{MPV}} 
& \multicolumn{6}{c|}{\textbf{AMPV}} \\
\cline{2-10}
& \textbf{1} & \textbf{5} & \textbf{10}
& \textbf{1} & \textbf{5} & \textbf{10} & \textbf{20} & \textbf{25} & \textbf{50} \\
\hline
\multicolumn{10}{|l|}{\textbf{Initialization phase}} \\
\hline
\multicolumn{1}{|l|}{ABox cloning}              
& 12.4  & 133.2  & 26.3 
& 9.6 & 29.8 & 33.9 & 43.7 & 33.0 & 59.0 \\
\multicolumn{1}{|l|}{Schema extraction}        
& 0.6  & 0.9  & 0.6 
& 0.6  & 0.8  & 0.6  & 0.6  & 0.7  & 0.6  \\
\multicolumn{1}{|l|}{Index creation}           
& 8.0 & 17.2 & 14.9
& 8.5  & 12.3 & 14.8 & 49.7 & 39.5 & 97.3 \\
\multicolumn{1}{|l|}{ABox expansion}           
& 3.9 & 13.8 & 24.5
& 4.0  & 12.5 & 26.9 & 64.3 & 72.8 & 167.2 \\
\multicolumn{1}{|l|}{Policy expansion}         
& 0.2 & 0.1 & 0.2
& 0.2  & 0.3  & 0.2  & 0.2  & 0.1  & 0.2  \\
\multicolumn{1}{|l|}{Auxiliary data structures}
& 2.0 & 5.6 & 11.4
& 2.0  & 8.7 & 10.0 & 21.1 & 26.8 & 61.9 \\
\hline
\multicolumn{10}{|l|}{\textbf{Iterative phase}} \\
\hline
\multicolumn{1}{|l|}{Violations identification} 
& 1.4 & 7.2 & 14.3 
& 1.5  & 7.7 & 14.9 & 32.4 & 40.9 & 91.5 \\
\multicolumn{1}{|l|}{Minimality check}           
& 46.1  & 2359.1  & 10428.9 
&  --  &  -- &  -- &  --  &  -- &  -- \\
\multicolumn{1}{|l|}{Violations deletion}           
& 4.5 & 25.7 & 84.4 
& 1.4  & 7.0 & 20.8 & 56.4  & 66.6 & 155.1 \\
\hline
\multicolumn{1}{|l|}{\textbf{Final phase}}
& 1.6  & 5.5  & 16.1  
& 1.3 & 8.2 & 14.6 & 26.7 & 25.3 & 56.3 \\
\hline
\multicolumn{1}{|l|}{\textbf{Total execution time}}
& 80.7 & 2568.3  & 10621.6 
& 29.1 & 84.3 & 136.7 & 295.1 & 305.7 & 689.1 \\
\hline
\end{tabular}

\caption{Execution time breakdown of the (A)MPV censor for increasing dataset
sizes, reporting both the initialization and iterative phases.
Times are reported in seconds.}
\label{tab:censor-summary}
\end{table*}

\newcommand{\evalTime}{$t_e$\xspace}
\newcommand{\numDel}{\#$_a$\xspace}
\newcommand{\numActiveEDs}{$\#_{eds}$}

\begin{table*}[t]
\centering
\setlength{\tabcolsep}{2pt}

\begin{tabular}{|c|c|c|c|c|c|c|c|c|c|c|}
\hline
 & & \multicolumn{9}{c|}{\textbf{Dataset size (N)}} \\
\cline{3-11}
&
& \multicolumn{3}{c|}{\textbf{MPV}} 
& \multicolumn{6}{c|}{\textbf{AMPV}} \\
\cline{3-11}
\textbf{Iteration} & \textbf{}
& \textbf{1} & \textbf{5} & \textbf{10} & \textbf{1} & \textbf{5} & \textbf{10} & \textbf{20} & \textbf{25} & \textbf{50} \\
\hline

\multirow{3}{*}{1}
& \evalTime &13.0  &497.9  &2438.8 &1.5  &7.4  &20.0  &52.6  &62.2  &143.7  \\
& \numDel &9487  &57106  &120163  &10752  &65678  &138570  &278145  &334897  &700465 \\
& \numActiveEDs &138  &138  &138  &138  &138  &138  &138  &138  &138  \\
\hline

\multirow{3}{*}{2}
& \evalTime &30.7  &1481.3  &6340.1 &1.1  &5.3  &10.9  &28.1  &34.8  &76.4  \\
& \numDel &3880  &25422  &53228 &3880  &25422  &53228  &105799  &127450  &259660 \\
& \numActiveEDs &67  &67  &67  &67  &67  &67  &67  &67  &67  \\
\hline

\multirow{3}{*}{3}
& \evalTime &8.4  &412.8  &1748.8 &0.3  &1.9  &4.7  &7.9  &10.4  &26.4  \\
& \numDel &0  &0  &0 &0  &0  &0  &0  &0  &0 \\
& \numActiveEDs &0  &0  &0  &0  &0  &0  &0  &0  &0  \\
\hline

\end{tabular}

\caption{Execution time (\evalTime, in seconds), number of deleted ABox assertions
(\numDel) and number of active EDs (\numActiveEDs) activated at the end of each iteration of the (A)MPV censor algorithm for different dataset sizes.}
\label{tab:iter-measures}
\end{table*}

\subsubsection{Results}
We first provide a comparison between the AMPV and MPV algorithms, and then we provide a more detailed analysis of the AMPV censor.
The analysis focuses on both the scalability of the approach and the iterative behavior of the censoring algorithm, with particular attention to factors that influence execution time and convergence.

\paragraph{AMPV-MPV comparison.}
For both AMPV and MPV censor, Table~\ref{tab:censor-summary} provides a breakdown of the total execution time across the main phases of the algorithm used for computing them.
A more fine-grained view of the iterative process is given in Table~\ref{tab:iter-measures}, which reports, for each iteration and dataset size, the execution time and the number of deleted ABox assertions.
This allows us to analyze how the workload and deletion activity evolve across iterations and how they contribute to the overall convergence of the algorithm.

\paragraph{Execution time.}
We now analyze the overall execution time and then provide a more detailed analysis by breaking it down into its main phases and internal components.

\begin{figure}[!t]
    \centering
    \includegraphics[width=0.45\textwidth]{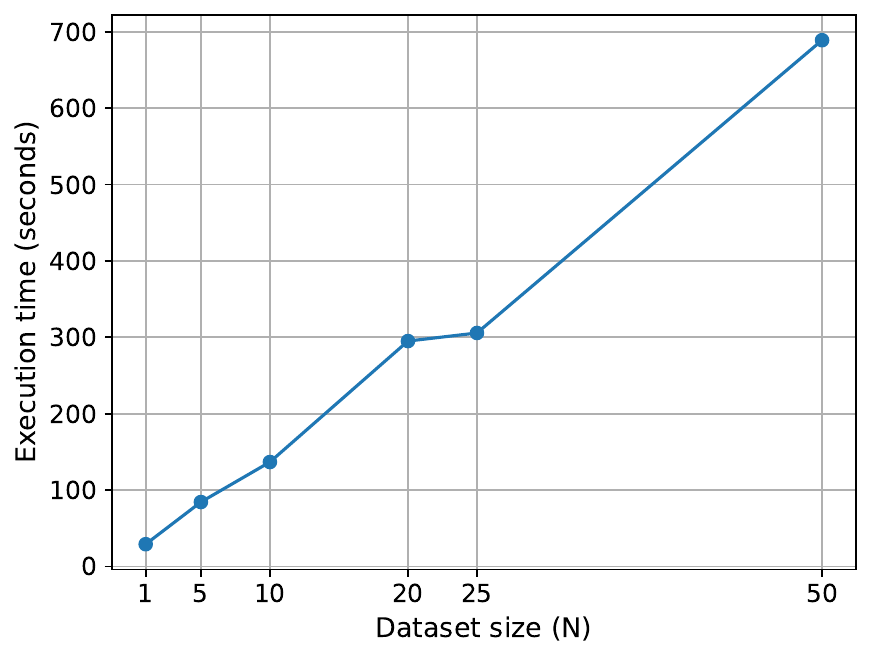}
    \caption{Execution time of the AMPV censor for increasing values of $N$.
    Each value of $N$ corresponds to an ABox of increasing size, as reported in Table~\ref{tab:core-ga-censor-summary}.}
    \label{fig:exec-time}
\end{figure}

Figure~\ref{fig:exec-time} reports the total time required to compute the AMPV censor as a function of the ABox size.
Interestingly, when compared with the growth of the expanded ABox size
(see Table~\ref{tab:core-ga-censor-summary}), the execution time exhibits a smooth and roughly proportional increase, suggesting stable scaling behavior across the considered dataset sizes.

To better understand the factors contributing to the overall execution time, we next analyze how the cost is distributed between the initialization phase and the iterative censorship phase, and we further investigate the internal breakdown of each phase.

\begin{figure}[t]
    \centering
    \includegraphics[width=0.45\textwidth]{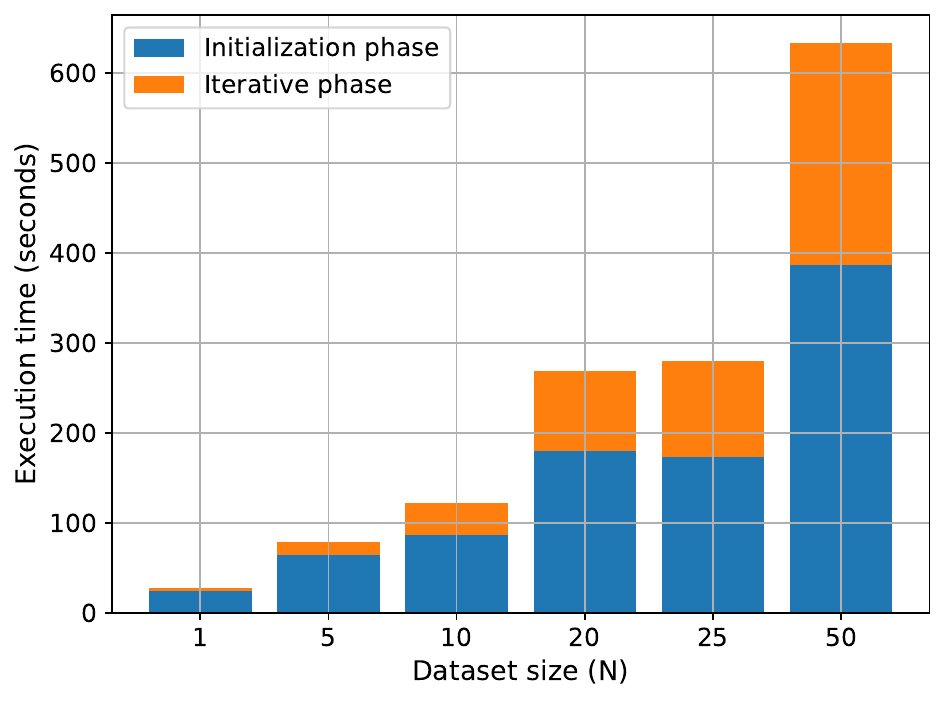}
    \caption{Execution time breakdown of the AMPV censor into initialization and iterative phases for different dataset sizes.}
    \label{fig:init-vs-iter}
\end{figure}

\subparagraph{Initialization vs.\ iterative phase.}
Figure~\ref{fig:init-vs-iter} reports the execution time of the AMPV censor by separating the initialization phase from the iterative censorship phase.
The initialization phase includes all preprocessing steps performed once before the fixpoint computation: database cloning, schema extraction, index creation, ABox and policy expansion, head expansion, and creation of temporary tables.
The iterative phase corresponds to the repeated identification and deletion of tuples violating the confidentiality policy until a fixpoint is reached.
As shown in the figure, for all dataset sizes, the initialization phase represents the dominant component of the overall execution time.
Although the cost of the iterative phase grows with the size of the input ABox, it remains consistently lower than that of the initialization phase.

\begin{figure}[t]
    \centering
    \includegraphics[width=0.45\textwidth]{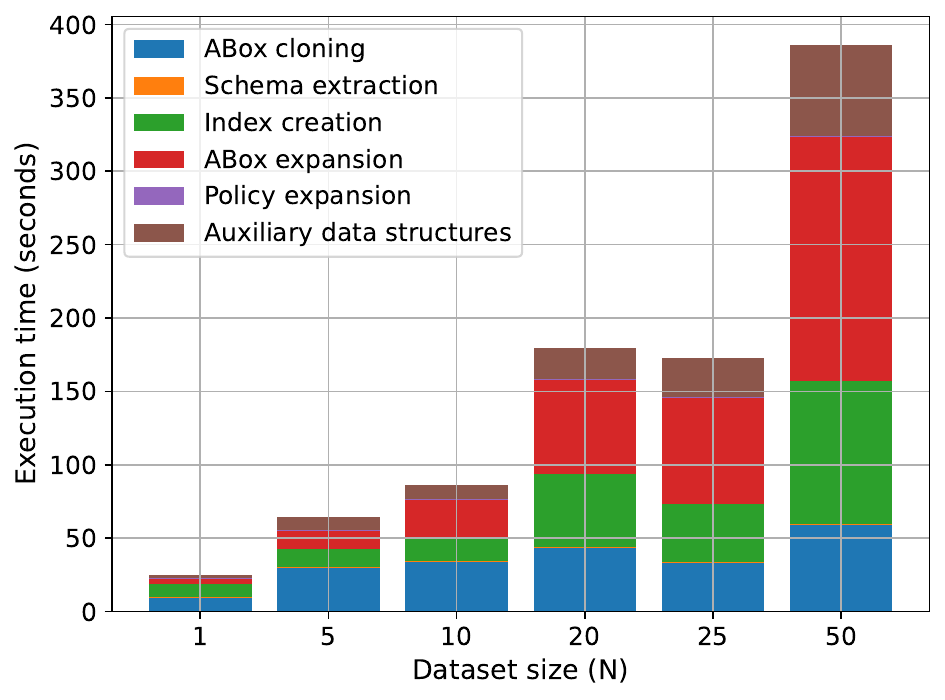}
    \caption{Execution time breakdown of the initialization phase of the AMPV censor for different dataset sizes.}
    \label{fig:init-breakdown}
\end{figure}

\subparagraph{Initialization phase breakdown.}
Figure~\ref{fig:init-breakdown} reports the execution time of the initialization phase of the AMPV censor, broken down into its main components.
The initialization phase of the AMPV censor algorithm is composed of several steps, including ABox cloning, schema extraction, index creation, ABox and policy expansion, and a step devoted to the materialization of auxiliary data structures.
This latter step comprises head expansion and splitting, body atom explosion, and the creation of temporary tables that are later used during the iterative censorship phase.
The results show that the initialization cost is mainly dominated by ABox cloning, ABox expansion, index creation, and auxiliary data structure materialization.
In contrast, schema extraction and policy expansion contribute only marginally to the overall initialization time.

\begin{figure}[t]
    \centering
    \includegraphics[width=0.45\textwidth]{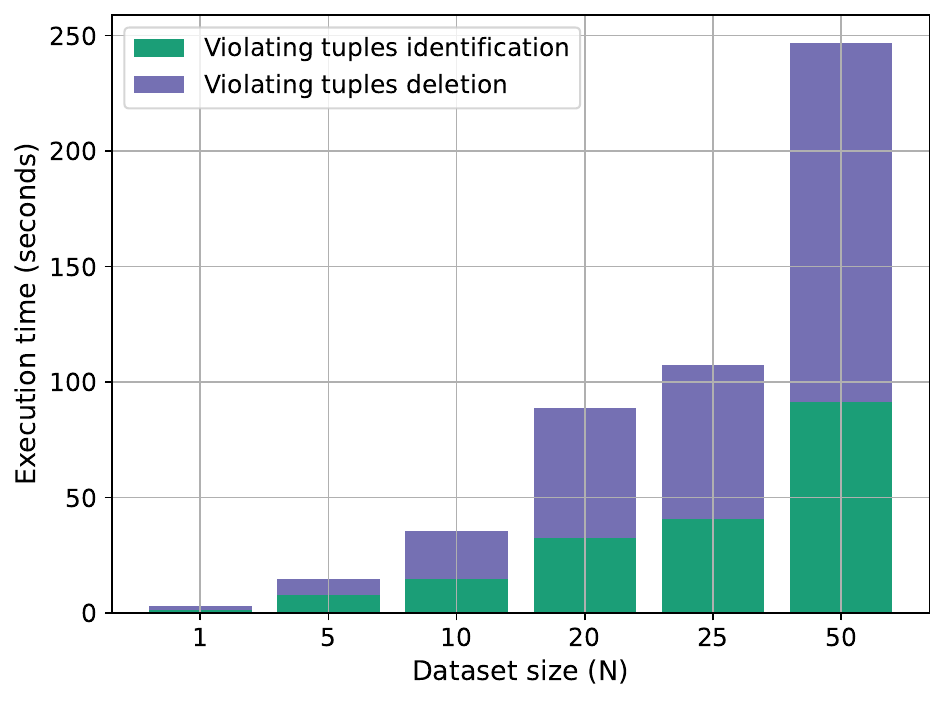}
    \caption{Execution time breakdown of the iterative phase of the AMPV censor into the \emph{violating tuple identification} and \emph{violating tuple deletion} steps for different dataset sizes.}
    \label{fig:iter-breakdown}
\end{figure}

\subparagraph{Iterative phase breakdown.}
Figure~\ref{fig:iter-breakdown} shows the execution time of the iterative phase of the AMPV censor, separated into the \emph{violating tuples identification} and \emph{violating tuples deletion} steps, for increasing dataset sizes.
The results indicate that for small datasets (N=1 and N=5), the two steps exhibit comparable execution times. However, as the dataset size increases, the deletion step progressively dominates the overall cost of the iterative phase, while the identification step contributes to a smaller but non-negligible portion of the total execution time.

\begin{figure}[t]
    \centering
    \includegraphics[width=0.45\textwidth]{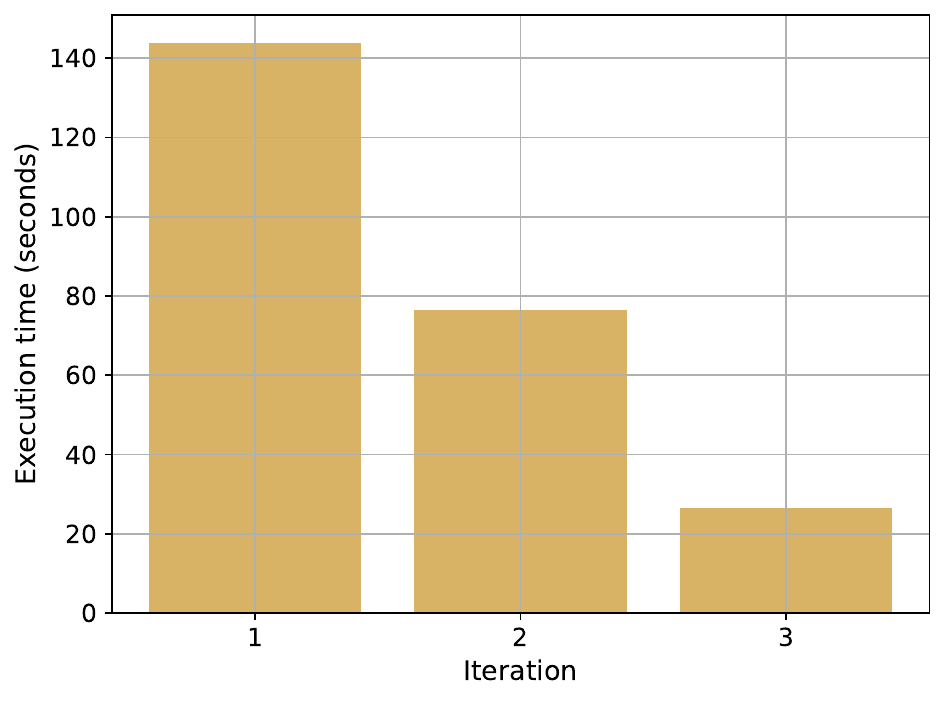}
    \caption{Execution time of each iteration of the iterative phase of the AMPV censor for $N = 50$.
    The last iteration corresponds to the fixpoint check and does not perform any deletions.}
    \label{fig:iter-cost}
\end{figure}

\subparagraph{Cost per iteration.}
Figure~\ref{fig:iter-cost} reports the execution time of each iteration of the iterative phase for the largest dataset ($N = 50$).
While the first iterations account for most of the computational effort, the final iteration performs no deletions and only serves to detect the fixpoint of the process, thus incurring in a negligible cost.
The same behavior is observed for all other dataset sizes, where the first iteration consistently represents the most expensive one.

\paragraph{Iterative behavior.}
In all the experimental configurations, the algorithm converges after a small and constant number of iterations, exactly three, independently of the size of the input ABox.
This behavior suggests that, in our experiments, convergence is primarily driven by the structure of the policy and by the dependencies among violations, rather than by the amount of data.
As a consequence, increasing the size of the ABox mainly affects the cost of each iteration, but not the total number of iterations required to reach the fixpoint.

\begin{figure}[H]
    \centering
    \includegraphics[width=0.45\textwidth]{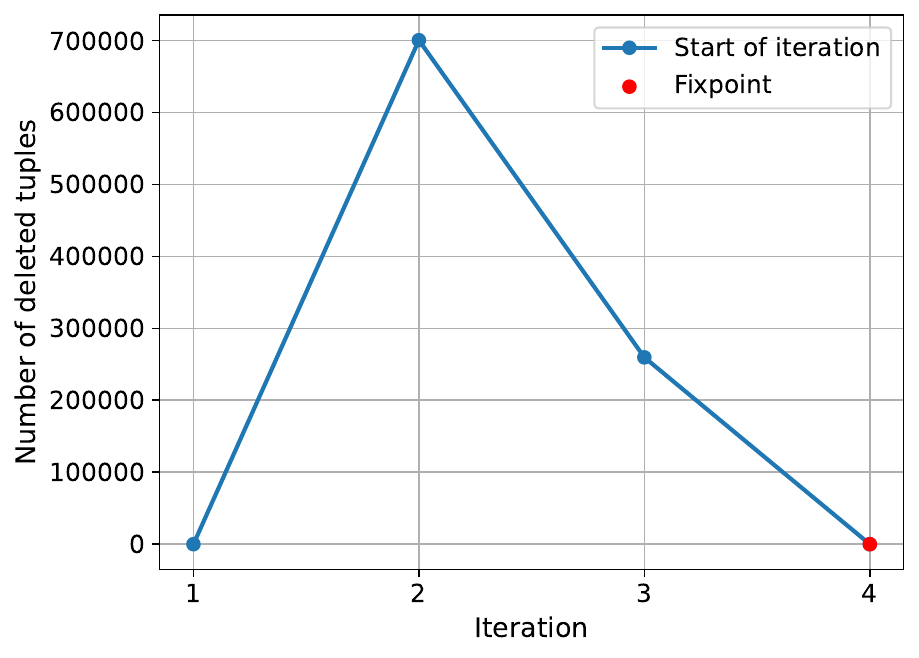}
    \caption{Number of ABox assertions deleted at each iteration of the AMPV censor
    for the largest dataset considered ($N = 50$).}
    \label{fig:tuple-del}
\end{figure}

\paragraph{Tuple deletions across iterations.}
Figure~\ref{fig:tuple-del} reports the number of ABox assertions deleted at each iteration for the largest dataset considered ($N = 50$).
Across all experimental configurations, a clear decreasing trend can be observed: the largest number of deletions occurs during the first iteration, and the number progressively decreases in subsequent iterations until the algorithm reaches a fixpoint.
This behavior reflects the dynamics of the censoring process.
Indeed, in our experiments, during the first iteration, the number of detected violations---and consequently the number of deletions---is maximal.
After this substantial clean-up step, the ABox is already significantly reduced, and fewer violations remain to be addressed.
As a result, subsequent iterations involve progressively fewer deletions.
The decreasing number of removed assertions indicates that the algorithm is converging.
Eventually, no further violations are produced, and a fixpoint is reached.

\begin{figure}[t]
    \centering
    \includegraphics[width=0.45\textwidth]{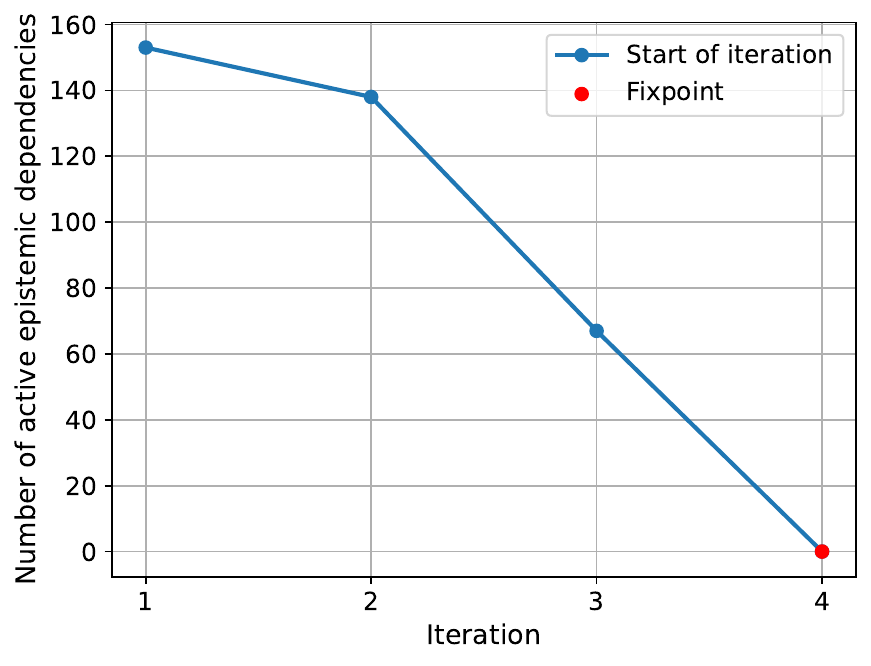}
    \caption{Number of active EDs at the start of each iteration of the AMPV censor for the largest dataset considered ($N = 50$).
    The red point represents the fixpoint reached at the end of the fifth iteration.}
    \label{fig:active-eds}
\end{figure}

\paragraph{Reduction of active EDs.}
Figure~\ref{fig:active-eds} reports the number of EDs that remain active at each iteration of the AMPV censor for the largest dataset considered.
The number of active dependencies decreases sharply across iterations, dropping from 153 in the first iteration to zero at convergence.
This behavior confirms the effectiveness of the optimization based on active EDs, which avoids re-evaluating dependencies that cannot be affected by recent
tuple deletions.
As a result, later iterations are significantly lighter than the initial ones, contributing to the overall scalability of the approach.

\paragraph{Scalability considerations.}
\newcommand{\Afinal}{\A_{\text{final}}}
When normalizing the execution time by the size of the final censored ABox, the average cost per retained tuple remains approximately stable across dataset sizes.
Formally, for each dataset size $N$, we compute the average cost per tuple as follows:
\[
\frac{T_{\text{total}}}{|\Afinal|},
\]
where $T_{\text{total}}$ denotes the total execution time of the censoring algorithm, measured in seconds, and $|\Afinal|$ is the number of ABox assertions remaining after the censorship process.
For presentation purposes, the resulting values are reported in milliseconds per tuple and summarized in Table~\ref{tab:ms-per-tuple}.

\begin{table}[ht]
\centering
\begin{tabular}{c c c}
\hline
$N$ & $|\Afinal|$ & ms/tuple \\
\hline
1  & 75\,054   & 0.38 \\
5  & 504\,036  & 0.16 \\
10 & 1\,113\,488  & 0.12 \\
20 & 2\,153\,670 & 0.13 \\
25 & 2\,677\,038 & 0.11 \\
50 & 5\,450\,941 & 0.12 \\
\hline
\end{tabular}
\caption{Average execution time per retained ABox tuple and size of the final censored ABox for increasing dataset sizes.}
\label{tab:ms-per-tuple}
\end{table}

In particular, for $N \geq 5$, the execution time ranges between $0.11$ and $0.16$ milliseconds per retained tuple.
This observation indicates that the increase in total runtime observed in the previous experiments is primarily explained by the growth of the processed data, rather than by a degradation in the efficiency of the censoring algorithm.
The higher normalized value observed for $N = 1$ can be explained by the relative impact of fixed overhead costs (database cloning, schema extraction, index creation, policy preprocessing, and the initialization of auxiliary data structures).
Although the initialization phase dominates the execution time across all dataset sizes, its cost does not grow proportionally with the size of the final censored ABox.
For small datasets, these fixed costs are amortized over a relatively small number of retained tuples, leading to a higher average cost per tuple.
As the dataset size increases, the number of retained assertions grows significantly, while the initialization overhead increases only moderately.
Consequently, when normalizing the execution time by $|\Afinal|$,
the fixed costs are distributed over a much larger number of tuples,
and the average cost per tuple stabilizes.

\subsubsection{Query answering results}
In this section, we report the results of the experimental evaluation of query answering over the censored ABoxes produced by the AMPV censor.
The analysis focuses on the scalability of query answering and on the impact of confidentiality enforcement on query execution performance, by investigating how query execution time and the number of returned answers evolve as the size of the censored ABox increases.
Table~\ref{tab:qa-summary} summarizes the average execution times and the number of returned answers for each query and each dataset size.
These tables provide a quantitative overview that complements the graphical analysis presented in the following paragraphs.




\begin{figure}[t]
    \centering
    \includegraphics[width=0.45\textwidth]{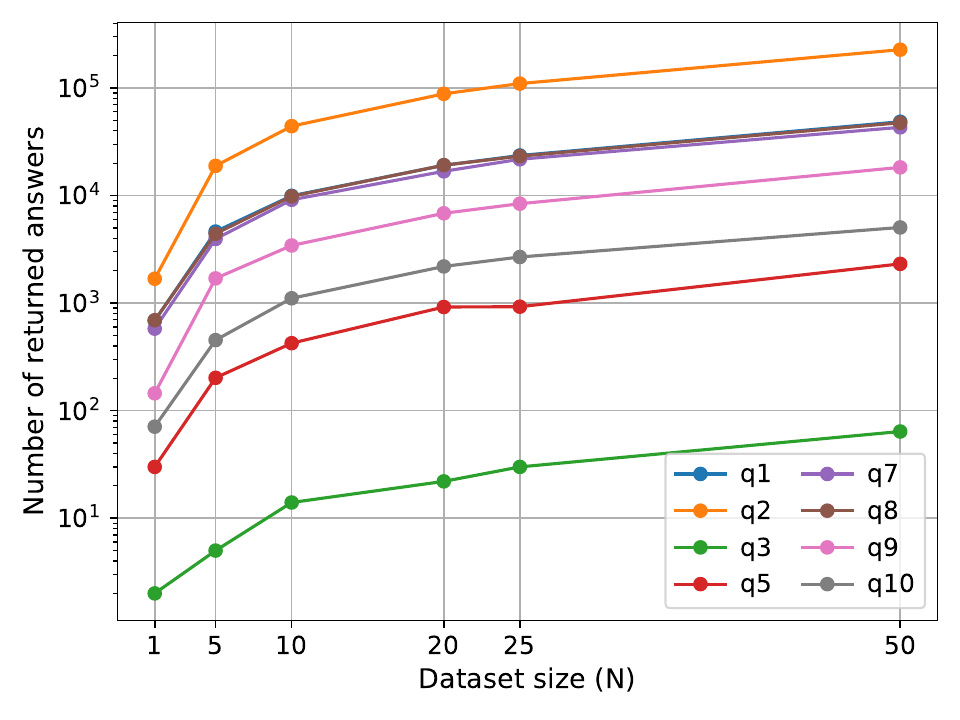}
    \caption{Number of returned answers for the benchmark queries with non-empty
    result sets as a function of the dataset size~$N$.
    A logarithmic scale on the y-axis is adopted to clearly visualize the different
    growth rates, since the result sizes vary significantly across queries.}
    \label{fig:qa-answers}
\end{figure}

\paragraph{Execution time.}
All query execution times are measured after a short warm-up phase.
For each query and each dataset size, one initial execution is performed and discarded, in order to mitigate the effects of caching, query compilation, and JVM warm-up.
Subsequently, each query is executed three times over the same censored ABox, and the reported execution time corresponds to the average of these runs.
This methodology allows us to obtain more stable and representative measurements of query answering performance.

Overall, the results show that query answering remains efficient and scalable across all dataset sizes considered.
For most queries, the execution time increases smoothly with the size of the censored ABox, reflecting the growth in the number of relevant tuples to be processed during query evaluation.
Even for the largest dataset ($N = 50$), query execution times remain within a few seconds, confirming that confidentiality enforcement through the AMPV censor does not introduce prohibitive overhead at query time.
A closer inspection reveals that the execution cost strongly depends on the structure of the query and on the cardinality of its result set.
Queries returning a large number of answers (i.e.\ queries~1,~2,~7,~8,~9 and~10) exhibit higher execution times.
In contrast, queries with limited or empty result sets (i.e.\ queries~4 and ~6) are evaluated very efficiently, often requiring less than one millisecond even for the largest ABoxes.
These observations indicate that the dominant factor influencing query execution time is the size of the intermediate and final results, rather than the presence of censorship.

\begin{figure}[t]
    \centering
    \includegraphics[width=0.45\textwidth]{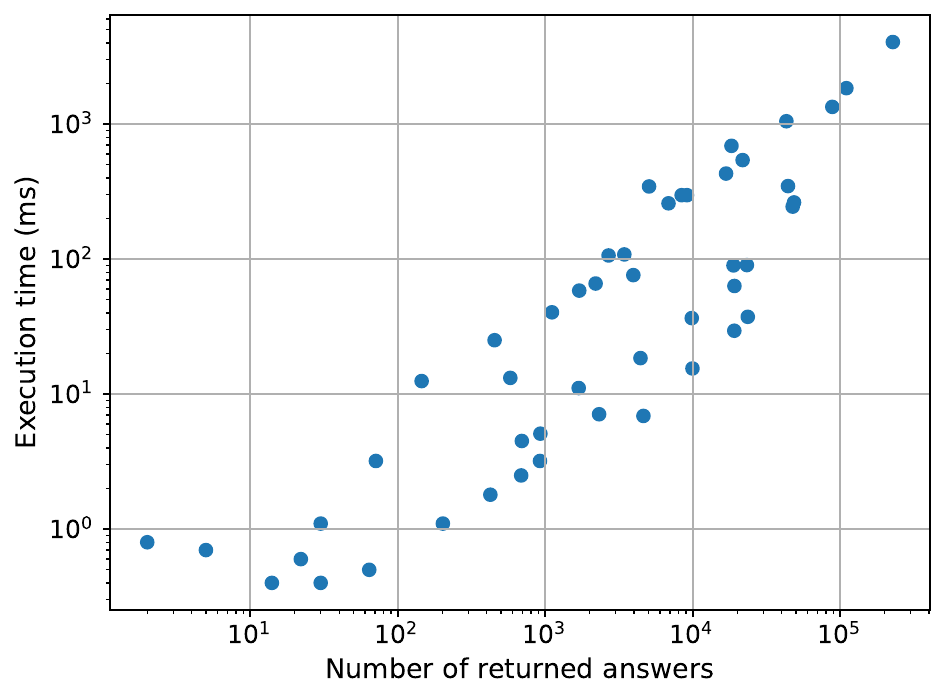}
    \caption{Relationship between the number of returned answers and the average execution time of the benchmark queries. Both axes are shown on a logarithmic scale.}
    \label{fig:qa-answers-time}
\end{figure}

\paragraph{Returned answers.}
We analyze the number of answers returned by each query as the size of the censored ABox increases.
Figure~\ref{fig:qa-answers} reports the results for the benchmark queries with non-empty result set, as a function of the dataset size.
queries~1,~2,~7,~8,~9, and~10 return large result sets whose cardinality grows steadily with $N$.
In particular, query~2 yields the largest number of answers, exceeding $220{,}000$ for $N = 50$.
This behavior reflects the presence of highly populated and strongly interconnected relations in the benchmark data.
Overall, query answering over censored ABoxes exhibits the same scaling trend as standard ontology-based query answering.
queries~3 and~5 produce moderate result sets that increase monotonically with $N$, but remain comparatively small due to their more restrictive patterns.
In contrast, queries~4 and~6 return no answers for any dataset size, as a consequence of the interaction between ontology semantics, data distribution, and the enforced confidentiality policy.

\paragraph{Relation between execution time and returned answers.}
To further investigate the factors underlying the observed execution times, we analyze the relationship between query execution time and the number of returned answers.
Figure~\ref{fig:qa-answers-time} reports the average execution time of the benchmark queries as a function of the size of their result sets.

The scatter plot highlights a clear positive relationship between the two quantities: queries returning a larger number of answers tend to require longer execution times, whereas queries with empty or very small result sets
are evaluated efficiently.
In particular, no queries with a small number of returned answers exhibit high execution times, nor do queries with large result sets achieve very low execution times.


\clearpage
\section{Benchmark queries}
\label{app:queries}
We report in Figure~\ref{fig:benchmark-queries} the complete set of SPARQL queries for OWL~2~QL used in the query answering evaluation, as provided by the OWL2Bench benchmark for OWL ontologies~\cite{owl2bench-queries}.

\begin{figure}[H]
    \centering

    \begin{tabular}{r@{\;\;}l}
        \text{q$_1$:} &
        \texttt{SELECT DISTINCT ?x ?y} \\
        & \texttt{WHERE \{?x :knows ?y\}} \\
        
        \text{q$_2$:} &
        \texttt{SELECT DISTINCT ?x ?y} \\
        & \texttt{WHERE \{?x :hasAlumnus ?y\}} \\
        
        \text{q$_3$:} &
        \texttt{SELECT DISTINCT ?x ?y} \\
        & \texttt{WHERE \{?x :isAffiliatedOrganizationOf ?y\}} \\
        
        \text{q$_4$:} &
        \texttt{SELECT DISTINCT ?x} \\
        & \texttt{WHERE \{?x :hasCollegeDiscipline :NonScience\}} \\
        
        \text{q$_5$:} &
        \texttt{SELECT DISTINCT ?x ?y} \\
        & \texttt{WHERE \{?x :hasCollaborationWith ?y\}} \\
        
        \text{q$_6$:} &
        \texttt{SELECT DISTINCT ?x ?y} \\
        & \texttt{WHERE \{?x :isAdvisedBy ?y\}} \\
        
        \text{q$_7$:} &
        \texttt{SELECT DISTINCT ?x} \\
        & \texttt{WHERE \{?x rdf:type :Faculty\}} \\
        
        \text{q$_8$:} &
        \texttt{SELECT DISTINCT ?x ?y} \\
        & \texttt{WHERE \{?x :hasSameHomeTownWith ?y\}} \\
        
        \text{q$_9$:} &
        \texttt{SELECT DISTINCT ?x ?y} \\
        & \texttt{WHERE \{} \\
        & \quad \texttt{?x rdf:type :Student.} \\
        & \quad \texttt{?x :isStudentOf ?y.} \\
        & \quad \texttt{?y :isPartOf ?z.} \\
        & \quad \texttt{?z :hasCollegeDiscipline :Engineering}\} \\
        
        \text{q$_{10}$:} &
        \texttt{SELECT DISTINCT ?s ?c} \\
        & \texttt{WHERE \{} \\
        & \quad \texttt{?s rdf:type :Student.} \\
        & \quad \texttt{?x rdf:type :Organization.} \\
        & \quad \texttt{?x :hasDean ?z.} \\
        & \quad \texttt{?z :teachesCourse ?c.} \\
        & \quad \texttt{?s :takesCourse ?c}\} \\
    \end{tabular}
    
    \caption{The 10 queries for OWL~2~QL from the OWL2Bench benchmark.}
    \label{fig:benchmark-queries}
\end{figure}

\clearpage
\section{Experimental policy}\label{app:policy}
\noindent
Figure~\ref{fig:epistemic-dependencies-experiments} presents the full set of EDs used in the experimental evaluation.
For the sake of readability, in these EDs we omit the universal quantification of the variables that are not existentially quantified.

\begin{figure}[H]
    \centering
    \[
    \begin{array}{r@{\;}l}
        \tau_1:&
        \K \,(\mathsf{knows}(x,y) \wedge \mathsf{Woman}(x))
        \rightarrow
        \K\,\exists z\,(\mathsf{hasSameHomeTownWith}(x,y)
        \wedge 
        \,\mathsf{hasDoctoralDegreeFrom}(x,z)) \\[1mm]
        
        \tau_2:&
        \K\,(\mathsf{hasSameHomeTownWith}(x,y) \wedge \mathsf{Woman}(x))
        \rightarrow
        \K \,(\mathsf{hasCollaborationWith}(x,y) \wedge \mathsf{knows}(x,y)) \\[1mm]
        
        \tau_3:&
        \K \,\mathsf{Woman}(x) \rightarrow
        \K\,\exists y,z\,(\mathsf{knows}(x,y)
        \wedge \,\mathsf{hasMasterDegreeFrom}(x,z)) \\[1mm]
        
        \tau_4:&
        \K(\mathsf{ScienceStudent}(x) \wedge \mathsf{hasAdvisor}(x,y))
        \rightarrow
        \K\,\exists z\,(\mathsf{Professor}(y))
        \wedge \, \mathsf{hasMasterDegreeFrom}(y,z))\\[1mm]
        
        \tau_5:&
        \K\,\exists y\,\mathsf{isAdvisedBy}(x,y)
        \rightarrow
        \K\,\mathsf{Woman}(x) \\[1mm]
        
        \tau_6:&
        \K\,\exists y\,(\mathsf{hasCollaborationWith}(x,y) \wedge \mathsf{Professor}(x))
        \rightarrow
        \K\,\exists z\,\mathsf{hasAdvisor}(z,x) \\[1mm]
        
        \tau_7:&
        \K\,\exists d\,(\mathsf{isPartOf}(x,z) \wedge \mathsf{hasCollegeDiscipline}(z,d))
        \rightarrow
        \K\,\exists p\,\mathsf{hasProgram}(x,p) \\[1mm]
        
        \tau_8:&
        \K\,(\mathsf{FullProfessor}(y) \wedge \mathsf{isAdvisedBy}(x,y))
        \rightarrow
        \K\,\exists z\,\mathsf{teachesCourse}(y,z) \\[1mm]
        
        \tau_9:&
        \K\,\mathsf{isAffiliatedOrganizationOf}(x,y)
        \rightarrow
        \K\,\mathsf{hasCollegeDiscipline}(x, Science) \\[1mm]
        
        \tau_{10}:&
        \K\,\exists y,u\,(\mathsf{teachesCourse}(x,y)
        \wedge \,
        \mathsf{hasMasterDegreeFrom}(x,u))
        \rightarrow
        \K\,\exists v\,\mathsf{hasDoctoralDegreeFrom}(x,v) \\[1mm]
        
        \tau_{11}:&
        \K\,\exists y\,(\mathsf{Faculty}(x) \wedge \mathsf{hasCollaborationWith}(x,y))
        \rightarrow
        \K\,\mathsf{Woman}(x) \\[1mm]
    \end{array}
    \]
    
    \caption{The 11 EDs used for our experiments.}
    \label{fig:epistemic-dependencies-experiments}
\end{figure}

\end{document}